\documentclass[journal]{IEEEtran}

\usepackage{hyperref}       
\usepackage{url}            
\usepackage{booktabs}       
\usepackage{amsfonts}   
\usepackage{amsmath}
\usepackage{algorithm, algorithmic}
\usepackage{bbm}
\usepackage{comment}
\usepackage[size=scriptsize, position=b]{subcaption}
\usepackage{graphicx}
\usepackage{amsthm}
\usepackage[noadjust]{cite}
\newtheorem{theorem}{Theorem}[section]

\newtheorem{lemma}[theorem]{Lemma}

\newlength\myindent
\title{Kernel-based Graph Learning from Smooth Signals: A Functional Viewpoint}

\author{Xingyue Pu,
        Siu Lun Chau,
        Xiaowen Dong,
        and Dino Sejdinovic
\thanks{Xingyue Pu and Xiaowen Dong are with the Oxford-Man Institute and the Department of Engineering Science, University of Oxford, Oxford OX2 6ED, UK (e-mail: xpu@robots.ox.ac.uk; xdong@robots.ox.ac.uk).}
\thanks{Siu Lun Chau and Dino Sejdinovic are with the Department of Statistics, University of Oxford, Oxford OX1 3LB, UK (e-mail: siu.chau@stats.ox.ac.uk; dino.sejdinovic@stats.ox.ac.uk).}
}

\begin{document}

\maketitle

\begin{abstract}
The problem of graph learning concerns the construction of an explicit topological structure revealing the relationship between nodes representing data entities, which plays an increasingly important role in the success of many graph-based representations and algorithms in the field of machine learning and graph signal processing. In this paper, we propose a novel graph learning framework that incorporates the node-side and observation-side information, and in particular the covariates that help to explain the dependency structures in graph signals. To this end, we consider graph signals as functions in the reproducing kernel Hilbert space associated with a Kronecker product kernel, and integrate functional learning with smoothness-promoting graph learning to learn a graph representing the relationship between nodes. The functional learning increases the robustness of graph learning against missing and incomplete information in the graph signals. In addition, we develop a novel graph-based regularisation method which, when combined with the Kronecker product kernel, enables our model to capture both the dependency explained by the graph and the dependency due to graph signals observed under different but related circumstances, e.g. different points in time. The latter means the graph signals are free from the $i.i.d.$ assumptions required by the classical graph learning models. Experiments on both synthetic and real-world data show that our methods outperform the state-of-the-art models in learning a meaningful graph topology from graph signals, in particular under heavy noise, missing values, and multiple dependency.
\end{abstract}

\begin{IEEEkeywords}
Graph learning, graph signal processing, kernel methods, functional viewpoint
\end{IEEEkeywords}

%
\IEEEpeerreviewmaketitle

\section{Introduction}
\label{sec:Introduction}


\begin{figure*}
    \centering
    \includegraphics[width = 0.85\linewidth]{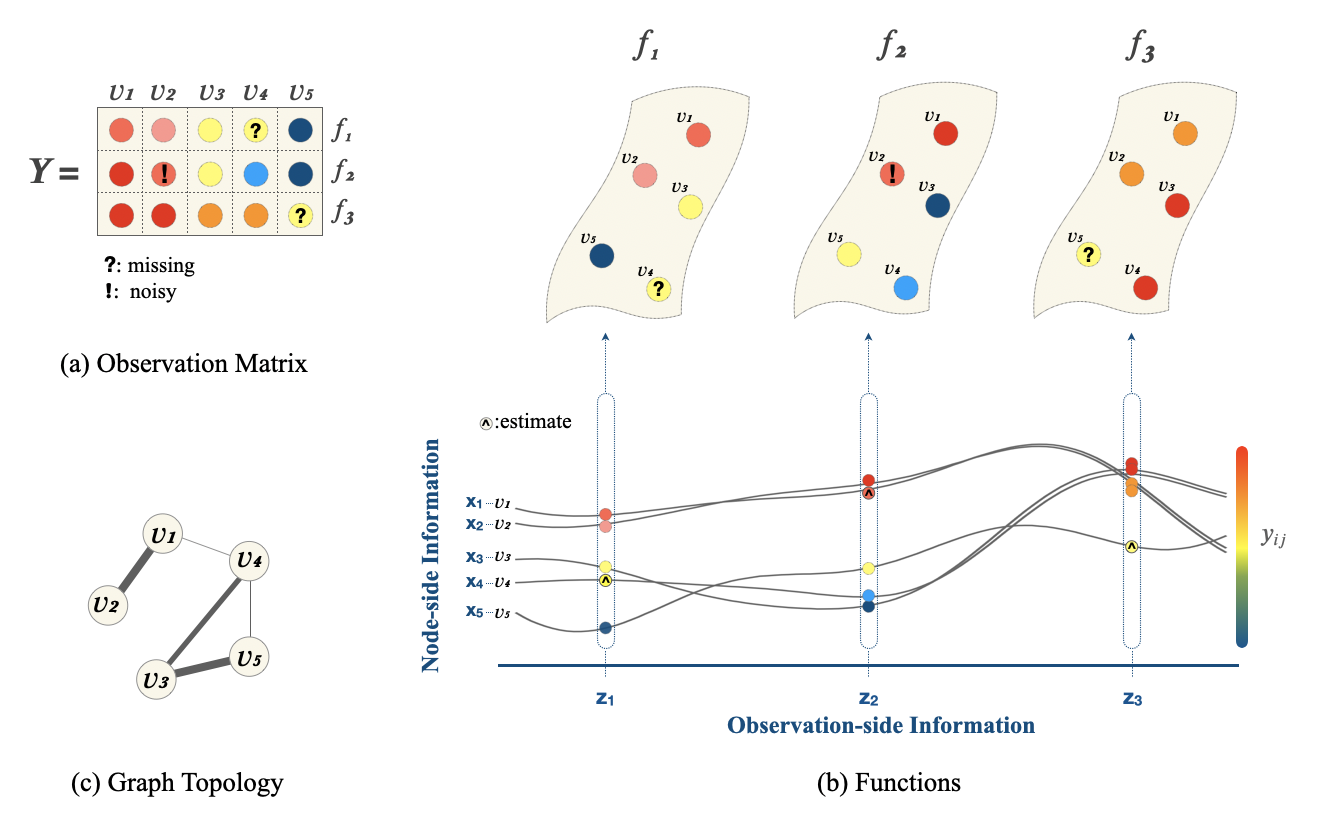}
    \caption{A functional viewpoint of graph learning: (a) The observation matrix with missing and noisy entries; (b) Each row of the observation matrix is modelled as samples obtained at a collection of fixed locations (considered as nodes in a graph) from an underlying function $f$ (top); Values at each node are determined by the underlying function, as well as node-side information $\mathbf{x}$ and observation-side information $\mathbf{z}$ (bottom); (c) The learned graph topology.}
    \label{fig:illustration}
\end{figure*}

Modelling based on graphs has recently attracted an increasing amount of interest in machine learning and signal processing research. On the one hand, many real-world data are intrinsically graph-structured, e.g. individual preferences in social networks or environmental monitoring data from sensor networks. This makes graph-based methods a natural approach to analysing such structured data. On the other hand, graphs are an effective modelling language for revealing relational structure in complex domains and may assist in a variety of learning tasks. 
For example, knowledge graphs improve the performance in semantic parsing and question answering \cite{berant2013semantic}. Despite their usefulness, however, a graph is not always readily available or explicitly given. The problem of graph learning therefore concerns the construction of a topological structure among entities from a set of observations on these entities. \par

Methodologies to learn a graph from the structured data include naïve methods such as $k$-nearest neighbours ($k$-NN), and approaches from the literature of probabilistic graphical models (PGMs) and more recently graph signal processing (GSP) and graph neural networks (GNNs). The basic idea of $k$-NN is to connect a node to $k$ other nodes with the smallest pairwise distances in terms of the observations \cite{Andoni2006, Muja2009a, Bentley1975a, Dong2011a}. In PGMs, a graph expresses the conditional dependence with edges between random variables represented by nodes \cite{lauritzen1996graphical}. 
The GSP literature, on the other hand, focuses on algebraic and spectral characteristics of the graph signals \cite{shuman2013emerging, Ortega2018,sandryhaila2013discrete}, which are defined as observations on a collection of nodes. The GSP-based graph learning methods (see \cite{mateos2019connecting,dong2019learning} for two recent reviews) further fall into two distinct branches, i.e. those based on the diffusion processes on graphs \cite{thanou2017learning, Pasdeloup2018,Shafipour2018,segarra2017network} and those based on smoothness measures of graph signals \cite{kalofolias2016learn,Dong2016a,egilmez2017graph,chepuri2017learning,berger2020efficient}. Very recently, GNNs have attracted a surging interest in the machine learning community which leads to a number of approaches to graph inference \cite{kipf2018neural,franceschi2019learning}.





While many of the above methods can effectively learn a meaningful graph from observations, there is a lack of consideration of the additional information, i.e. node-side or observation-side covariates, which may be available for the task at hand. Those covariates that provide valuable side information should be integrated into the graph learning framework. Taking an example of measuring temperature records in different locations in a country, where nodes represent weather stations, the latitude, longitude and altitude of each station are useful node-side information. One major benefit is to lessen the reliance of the above models on the quality of the observations. Heavily corrupted or even missing records can be predicted from the relationship between the observations and the side information, which in turn helps improve the efficiency in graph inference. \par

Furthermore, although node-side dependency is inherently accounted for in the process of graph learning, the observation-side dependency is largely ignored in the literature.  One example are temperature records collected at different timestamps, which could largely affect the evaluation of the strength of relation between stations. 
Another example is that of a recommender system, where the item ratings collected from different individuals are largely affected by the social relationship between them.  \par

To tackle the above issues, we revisit the graph signal observations from a functional viewpoint and propose a framework for learning undirected graphs by considering additional covariates on both the node- and observation-side. This allows us to capture dependency structure within the graph signals which leads to more effective graph and signal recovery. More specifically, as shown in Figure \ref{fig:illustration}, 
the $ij$-th entry of the graph-structured data matrix $\mathbf{Y} \in \mathbb{R}^{n \times m}$, which contains $n$ graph signals collected on $m$ nodes, can be viewed as some potentially noisy or missing observation of $f_i(j)$, i.e. the $i$-th function evaluated at $j$-th node. To model the node-side information, we introduce a covariate $\mathbf{x} \in \mathcal{X}$ that can explain the variations in a graph signal, e.g. a vector that contains the latitude, longitude and altitude of stations in the aforementioned temperature example. 
To model the observation-side information, we also introduce a generic covariate $\mathbf{z} \in \mathcal{Z}$. For example, $\mathbf{z}$ could be the timestamp at which the temperature record is collected. Observation-side dependency hence arises due to $f_i$ depending on $\mathbf{z}_i$. 
Combining the two, the function underlying the graph signals takes the form of 
$f_i(j) = f(\mathbf{z}_i, \mathbf{x}_j)$.


Specifically, we define the function $f: \mathcal{Z} \times \mathcal{X} \rightarrow \mathbb{R}$ in a reproducing kernel Hilbert space (RKHS) with a product kernel $\kappa_{\otimes} = \kappa_{\mathcal{Z}} \otimes \kappa_{\mathcal{X}}$ on $\mathcal{Z} \times \mathcal{X}$. At the same time, the two-side dependency in $\mathbf{Y}$ is encoded in a Kronecker product of two graph Laplacian matrices $\mathbf{L}_{\otimes} = \mathbf{L}_x \otimes \mathbf{L}_z$, where $\mathbf{L}_x$ represents the connectivity between nodes to be learned and $\mathbf{L}_z$ represents the observation-side dependency, essentially a nuisance dependency for the graph learning problem. We assume $\mathbf{L}_z$ can be captured by evaluating $\kappa_{\mathcal{Z}}$ at the observation-side covariates $\mathbf{z}$. Our key contribution is the Kernel Graph Learning (\textbf{KGL}) framework, which allows us to infer $\mathbf{L}_x$ by jointly learning the function $f$ and optimising for a novel Laplacian quadratic form that effectively expresses the smoothness of $\mathbf{Y}$ over $\mathbf{L}_{\otimes}$. \par

In addition, we provide several extensions of \textbf{KGL} for the scenario of a partially observed $\mathbf{Y}$ with known missing value positions, and that of observations without either node-side or observation-side information. The learning problem is effectively solved via a block coordinate descent algorithm, which has a theoretical guarantee of convergence. We show that \textbf{KGL} can effectively recover the groundtruth graph from the two-side dependent data and outperform the state-of-the-art smoothness-based graph learning methods in both synthetic and real-world experiments. \par


In summary, the main contributions of our work are as follows:
\begin{itemize}
    \item A novel graph-based regularisation based on a smoothness measure of dependent graph signals over the Kronecker product of two graph Laplacian matrices;
    \item A graph learning framework that integrates node- and observation-side covariates from a functional viewpoint;
    \item An efficient method for denoising and imputing missing values in the observed graph signals as a byproduct of the graph learning framework. 
\end{itemize}


\section{Related Work}
\label{sec: related_work}

In this section, we survey the classical methods of learning a graph from a number of different perspectives. 
From each perspective, we highlight the most related work that considers one or more aspects of the 1) node-side information, 2) observation-side dependency, and 3) noisy and missing data. \par

\subsection{$k$-Nearest Neighbours Methods}

The $k$-nearest neighbours ($k$-NN) connects a node to $k$ other nodes with the smallest pairwise distances in terms of the observations. It is flexible with different choices of distance metrics and yet heuristic since the neighbourhood search is based on pairwise comparison of observations on nodes. The majority of the $k$-NN variants focuses on fast approximate search algorithms (ANN) \cite{Bentley1975a, Andoni2006, Muja2009a,  Dong2011a} and recent variants apply deep reinforcement learning to explicitly maximise search efficiency \cite{Xiong2017,Zoph2019}. By comparison, the model-based methods, e.g. PGMs and GSP, directly integrate global properties of the observations into learning objectives. \par

\subsection{Probabilistic Graphical Models}

In the field of PGMs, the inverse covariance matrix $\Theta$ is often regarded as an undirected graph that parameterises the joint distribution of random variables representing nodes. There is a rich literature on effective algorithms to estimate a sparse $\Theta$ from Gaussian-distributed data by solving an $\ell_1$-regularised log-likelihood maximisation \cite{Banerjee2008ModelData, DAspremont2008, Mazumder2012, Hsieh2011, Zhang2014}, including the widely used graphical Lasso \cite{Friedman2008a} and G-ISTA \cite{Guillot2012}. The recent state-of-the-art algorithm BigQUIC \cite{Hsieh2013e} scales to millions of nodes with performance guarantees. 
Besides computational improvements, models based on attractive Gaussian Markov Random Fields (GMRFs) \cite{Lake2010a, Slawski2015a, Egilmez2017b, Deng2020} further restrict the off-diagonal entries of $\Theta$ to be non-positive, which is equivalent to learning the Laplacian matrix of the corresponding graph with non-negative edge weights. 
The most related extensions of the graphical Lasso were proposed in \cite{stegle2011efficient,greenewald2019tensor}, which simultaneously learn two dependency structures in the matrix-variate Gaussian data. While their work focuses on estimating covariance matrices, our work focuses on recovering a graph topology from data. 

\subsection{Structural Equation Models}

Structural equation models (SEMs) are another type of models (similar to PGMs) that is widely used to learn a directed acyclic graph (DAG) that encodes the conditional dependence of random variables \cite{zheng2018dags, Giannakis2018a, pamfil2020dynotears}. Based on SEMs, the authors in \cite{ioannidis2018semi} proposed a block coordinate descent algorithm to solve the joint optimisation problem of denoising the data and learning a directed graph. The joint learning framework is further extended to time series \cite{ioannidis2019semi}, where the structural vector autoregressive models (SVARMs) replace the linear SEMs to handle temporal dependency. The main difference from our work is that their denoising function is an identity mapping without side information as covariates. The work in \cite{pamfil2020dynotears} also considers the temporal dependency in learning a DAG with SVARMs, but does not consider the denoising scenario. \par

\subsection{Graph Signal Processing}

In the context of GSP, every observation on a collection of nodes is defined as a graph signal. GSP-based graph learning models have seen an increasing interest in the literature \cite{mateos2019connecting,dong2019learning} and further fall into two distinct branches. The first branch assumes graph signals are outcomes of diffusion processes on graphs and reconstructs a graph from signals according to the diffusion model \cite{thanou2017learning,Pasdeloup2018,Shafipour2018,segarra2017network}. The other branch constructs a graph by promoting the global smoothness of graph signals, which is defined by the Laplacian quadratic form \cite{kalofolias2016learn,Dong2016a} or more generally via total variation \cite{berger2020efficient}. Smoothness-based methods are related to GMRFs by recognising that the Laplacian quadratic form is closely related to the log-likelihood of the precision matrix defined as the graph Laplacian. Our work can be regarded as an extension to smoothness-based graph construction. \par

In the literature of smoothness-based GSP graph learning, the authors in \cite{Dong2016a, berger2020efficient} adopt a two-step learning algorithm to learn an undirected graph while denoising graph signals. They simply assume an identity mapping between the actual graph signals and noisy observations, which is different from our work that considers side information. The most related work is proposed in \cite{venkitaraman2019predicting}, which uses kernel ridge regression with observation-side covariates to infer graph signals. However, their work mainly focuses on data prediction and graph learning is only a byproduct in their approach. In Section \ref{sec: syn_exp_general_settings}, we will show, both theoretically and empirically, that their method 
uses a smoothness term that imprecisely incorporates the observation-side dependency in the learned graph structure, leading to an inferior performance in learning a graph. \par

In terms of the observation-side dependency, there exist some GSP graph learning models that consider temporal dependency in graph signals. A so-called spatiotemporal smoothness was proposed in \cite{Liu2019a,liu2019graph} to transform the graph signals using a temporally weighed difference operator. If every timestamp is equally important, the operator is equivalent to a prepossessing step to make the time series observed on each node stationary.
It should be noted that there is another branch of research assuming that the temporal dependency in graph signals originates in the dynamic changes in the edges \cite{Kalofolias2017a,Yamada2019e}, and therefore the problem is formulated as learning a dynamic series of graphs, which is different from the goal of our paper. \par

\subsection{Graph Neural Networks}

A new branch of graph learning models is developed from the perspective of GNNs. Essentially, GNNs discover the patterns in graph-structured data in a hierarchical manner \cite{zhang2020deep, wu2020comprehensive, bronstein2017geometric}. 
The activations at intermediate layers, e.g. the $l$-th layer, can be interpreted as a new representation for the nodes in the embedding space that incorporates the information from a specifically defined neighbourhood of the nodes.
The authors in \cite{Hadziosmanovic2019a, kipf2018neural} thus defined the strength of connectivity between nodes $i$ and $j$ based on the pairwise similarity of their embeddings $h_i^{(l)}$ and $h_j^{(l)}$ at the $l$-th layer of the GNN architecture. The authors in \cite{wu2020connecting} extended this method to construct a directed graph 
in the process of training a GNN that deals with time series data.
The main goal of these methods is to improve the performance of node-related tasks (e.g. classification or prediction) and graph learning is only a byproduct, whose performance is often not guaranteed. The recent works in \cite{Jiang2019,Yang2019,Chen2019,franceschi2019learning} start to incorporate an additional loss for recovering graphs while training the GNNs. However, a significant limitation of most GNN-based methods is that they typically require a large volume of training data and the learned connectivity is often less explainable compared to PGM and GSP methods. \par

\section{Preliminaries}
\label{sec:preliminary}

\subsection{Smoothness-Based GSP Graph Learning}
\label{sec:preliminary_GSP_GL}

Observing a data matrix $\mathbf{Y} \in \mathbb{R}^{n \times m}$ whose $ij$-th entry $y_{ij}$ corresponds to the observation on the $j$-th node in the $i$-th graph signal, we are interested in constructing an undirected and weighted graph $\mathcal{G} = \{\mathcal{V}, \mathcal{E}, \mathbf{W}\}$. The node set $\mathcal{V}$ represents a collection of variables, where $|\mathcal{V}| = m$. The edge set $\mathcal{E}$ represents the relational structure among them to be inferred. The structure is completely characterised by the weighted adjacency matrix $\mathbf{W}$ whose $jj'$-th entry is $w_{jj'}$. If two nodes $j$ and $j'$ are connected by an edge $e_{jj'} \in \mathcal{E}$ then $w_{jj'} > 0$, else if $e_{jj'} \notin \mathcal{E}$ then $w_{jj'} = 0$. The graph Laplacian matrix is defined as $\mathbf{L} = \text{diag}(\mathbf{W}\boldsymbol{1}) - \mathbf{W}$, where $\boldsymbol{1}$ denotes the all-one vector. $\mathbf{L}$ and $\mathbf{W}$ are equivalent and complete representations of the graph on a given set of nodes. \par

In the literature of GSP, one typical approach of constructing a graph from $\mathbf{Y}$ is formulated as minimising the variation of signals on graphs as measured by the Laplacian quadratic form\footnote{We acknowledge that the conventional form of the Laplacian quadratic in GSP literature is $\text{Tr}(\mathbf{Y}^{\top} \mathbf{L} \mathbf{Y})$, where each column of $\mathbf{Y}$ corresponds to a graph signal. In our case, $\mathbf{Y}$ has two-side dependency such that either a column or a row may be regarded as a graph signal. The term $\text{Tr}(\mathbf{Y} \mathbf{L} \mathbf{Y}^{\top})$ measures the smoothness of row vectors over a column graph. This formulation is however consistent with the statistical modelling convention where each column in $\mathbf{Y}$ is often regarded as a random variable and the graph of main interest is the column graph.} \cite{shuman2013emerging,kalofolias2016learn,Dong2016a}:
\begin{equation}
    \label{eq: smoothnessL}
    \min_{\mathbf{L} \in \mathcal{L}} ~ \text{Tr}(\mathbf{Y} \mathbf{L} \mathbf{Y}^{\top}) + \lambda \Omega(\mathbf{L}) 
\end{equation}%
where $\mathcal{L} = \{\mathbf{L} | \mathbf{L}\mathbf{1} = 0,  \mathbf{L}_{jj'} = \mathbf{L}_{j'j} \leq 0, \forall j \neq j' \}$ defines the space of valid Laplacian matrices, and $\Omega(\mathbf{L})$ is a regularisation term with a hyperparameter $\lambda > 0$. Equivalently, the problem can be formulated using the weighted adjacency matrix $\mathbf{W}$ such that
\begin{equation}
\min_{\mathbf{W} \in \mathcal{W}} ~ \frac{1}{2} \sum_{i = 1}^n \sum_{j, j'} w_{jj'} (y_{ij} - y_{ij'})^2 + \lambda \Omega(\mathbf{W}) 
\end{equation}%
where $\mathcal{W} = \{\mathbf{W} | \text{diag}(\mathbf{W}) = \mathbf{1}, w_{jj'} = w_{jj'} \geq 0, \forall j \neq j' \}$ defines the space of valid weighted adjacency matrices. Popular choices of regularisation include the sum barrier $\Omega(\mathbf{L}) = ||\mathbf{L}||_F^2$ and  $\Omega(\mathbf{L}) = |\text{Tr}(\mathbf{L}) - m|$ (often added as a constraint such that $\text{Tr}(\mathbf{L}) = m$) to prevent trivial solutions where all edge weights are zero and meanwhile controlling the variations of edge weights \cite{Dong2016a}, or the log-barrier $\Omega(\mathbf{W}) = \mathbf{1}^{\top} \log(\mathbf{W}\mathbf{1})$ to prevent isolated nodes and promote connectivity \cite{kalofolias2016learn}. 

With a fixed Frobenius norm for $\mathbf{Y}$, a small value of the objective in Eq.(\ref{eq: smoothnessL}) implies that $\mathbf{Y}$ is smooth on $\mathcal{G}$ in the sense that neighbouring nodes have similar observations. The authors in \cite{Dong2016a} further propose a probabilistic generative model of the noise-free smooth observations $\mathbf{Y} = [\mathbf{y}_1, \mathbf{y}_2, \dots, \mathbf{y}_n]^{\top}$ such that
\begin{equation}
\label{eq: generative_Dong}
    \mathbf{y}_i \overset{i.i.d.}{\sim} \mathcal{N}(0, \mathbf{L}^\dagger),\quad i = 1,2,\dots,n
\end{equation}%
where $(\cdot)^\dagger$ denotes the Moore-Penrose pseudo-inverse of a matrix. 
This leads to a graph learning framework which solves an optimisation problem similar to Eq.(\ref{eq: smoothnessL}). \par


\subsection{Kronecker Product Kernel Regression}
\label{sec:preliminary_RKHS}

Taking a functional viewpoint on the generation of graph-structured data matrix, we can make use of the well-studied formalism of Kronecker product kernel ridge regression to infer the latent function \cite{saatcci2012scalable}. Specifically, we consider $f:\mathcal{Z} \times \mathcal{X} \rightarrow \mathbb{R}$ to be an element of a reproducing kernel Hilbert space (RKHS) corresponding to the product kernel function $\kappa_\otimes  = \kappa_{\mathcal{Z}} \otimes \kappa_{\mathcal{X}}$ on $\mathcal{Z} \times \mathcal{X}$, where $\otimes$ denotes the Kronecker product.\par

A kernel function can be expressed as an inner product in a corresponding feature space, i.e. $\kappa_{\mathcal{Z}}(\mathbf{z}_i, \mathbf{z}_{i'}) = \langle \phi_{\mathcal{Z}}(\mathbf{z}_i), \phi_{\mathcal{Z}}( \mathbf{z}_{i'})\rangle_{\mathcal{H}_{\mathcal{Z}}}$ where $\phi_{\mathcal{Z}}: \mathcal{Z} \rightarrow {\mathcal{H}_{\mathcal{Z}}}$ and $\kappa_{\mathcal{X}}(\mathbf{x}_i, \mathbf{x}_{i'}) = \langle \phi_{\mathcal{X}}(\mathbf{x}_i), \phi_{\mathcal{X}}( \mathbf{x}_{i'})\rangle_{\mathcal{H}_{\mathcal{X}}}$, where $\phi_{\mathcal{X}}: \mathcal{X} \rightarrow {\mathcal{H}_{\mathcal{X}}}$. An explicit representation of feature maps $\phi_{\mathcal{X}}$ and $\phi_{\mathcal{Z}}$ is not necessary and the dimension of mapped feature vectors could be high and even infinite. By the representer theorem, the function $f$ that fits the data $\mathbf{Y}$ takes the form
\begin{equation}
\begin{split}
\label{eq: RKHSfunction}
    f(\mathbf{z}, \mathbf{x}) = \sum_{i = 1}^n \sum_{j=1}^m a_{ij} \kappa_{\mathcal{Z}}(\mathbf{z}_i, \mathbf{z})\kappa_{\mathcal{X}}(\mathbf{x}_j, \mathbf{x})
\end{split}
\end{equation}%
where $a_{ij} \in \mathbb{R}$ are the coefficients to be learned, and the estimated value $\hat{y}_{ij} = f(\mathbf{z}_i, \mathbf{x}_j)$. Denoting the corresponding kernel matrices as $\mathbf{K}_z$ and $\mathbf{K}_x$, where $(\mathbf{K}_z)_{ii'} = \kappa_{\mathcal{Z}}(\mathbf{z}_i, \mathbf{z}_{i'})$ and $(\mathbf{K}_x)_{jj'} = \kappa_{\mathcal{X}}(\mathbf{x}_j, \mathbf{x}_{j'})$, we have the matrix form
\begin{equation}
\label{eq: RKHSfunction_matrix}
    \mathbf{\hat{Y}} = \mathbf{K}_z \mathbf{A} \mathbf{K}_x
\end{equation}%
where $\mathbf{\hat{Y}}$ is an approximation to $\mathbf{Y}$ and the coefficient matrix $\mathbf{A} \in \mathbb{R}^{n \times m}$ has the $ij$-th entry as $a_{ij}$. We assume the observation $\mathbf{Y}$ is a noisy version of $\mathbf{\hat{Y}}$, where the noise is $i.i.d.$ normally distributed random variables such that $\epsilon_{ij} = \mathbf{Y}_{ij} - \mathbf{\hat{Y}}_{ij} \sim \mathcal{N}(0, \sigma^2_\epsilon)$, where $\sigma^2_\epsilon$ measures the noise level. This leads to a natural choice of the sum of squared errors as the loss function.    \par

A standard Tikhonov regulariser is often added to the regression model to reduce overfitting and penalise complex functions, which is defined in our case as 
\begin{equation}
\begin{split}
    ||f||^2_{\mathcal{H}_{\otimes}} & = \text{vec}(\mathbf{A})^{\top} (\mathbf{K}_x \otimes \mathbf{K}_z) \text{vec}(\mathbf{A}) \\
    & = \text{Tr}( \mathbf{K}_z \mathbf{A} \mathbf{K}_x \mathbf{A}^{\top})
\end{split}
\end{equation}%
where $\text{vec}(\cdot)$ is the vectorisation operator for a matrix. 
We now arrive at the following optimisation problem to infer the function $f(\mathbf{z}, \mathbf{x})$ that approximates the observation matrix $\mathbf{{Y}}$ such that
\begin{equation}
\label{eq: kronecker_regression}    
\min_{\mathbf{A}} ~ ||\mathbf{Y} - \mathbf{K}_z \mathbf{A}  \mathbf{K}_x||_F^2  + \lambda \text{Tr}(\mathbf{K}_z \mathbf{A}  \mathbf{K}_x \mathbf{A}^{\top}) 
\end{equation}%
where the hyperparameter $\lambda > 0$ controls the penalisation of the complexity of the function to be learned. \par


To have a better understanding of how this model is expressive for the two-side dependency, we show that the objective in Eq.(\ref{eq: kronecker_regression}) can be derived from a Bayesian viewpoint. 
In the vector form, i.e. $\mathbf{a} = \text{vec}(\mathbf{A})$ and $\mathbf{y} = \text{vec}(\mathbf{Y})$, we assume that both the data likelihood $p(\mathbf{y}| \mathbf{a})$ and the prior $p(\mathbf{a})$ follow a Gaussian distribution:
\begin{subequations}
\begin{align}
    \mathbf{y}| \mathbf{a} & \sim \mathcal{N}((\mathbf{K}_x \otimes \mathbf{K}_z) \mathbf{a}, \sigma^2_{\epsilon}\mathbf{I}_{nm})  \label{eq: y_generative_data_likelhood}, \\ 
    \mathbf{a} & \sim \mathcal{N}(\mathbf{0}_{nm}, \mathbf{K}_x^{\dagger} \otimes \mathbf{K}_z^{\dagger}),  \label{eq: y_generative_a_prior}%
\end{align}
\label{eq: y_generative_all}%
\end{subequations}%
where $\mathbf{0}_{mn}$ is a zero-vector of length $mn$. Notice that $\mathbf{K}_x$ and $\mathbf{K}_z$ (and their Kronecker product) can be either singular or non-singular matrices, depending on the kernel choice. For the sake of simplicity, we use the pseudo-inverse notation throughout the paper. Now, the marginal likelihood of $\mathbf{y}$ is
\begin{equation}
\label{eq: y_generative_marginal}
    \mathbf{y} \sim \mathcal{N}(\mathbf{0}_{mn}, \mathbf{K}_x \otimes \mathbf{K}_z +  \sigma^2_{\epsilon}\mathbf{I}_{nm})
\end{equation}%
from which we can see that the covariance structure of $\mathbf{y}$ can be understood as a combination of $\mathbf{K}_x$ and $\mathbf{K}_z$. Specifically, in the noise-free scenario where $\sigma^2_\epsilon = 0$, the covariance matrix over the rows of $\mathbf{Y}$ is 
\begin{equation}
\label{eq:row_cov}
    \text{Cov}_r[\mathbf{Y}] = E[\mathbf{Y}^{\top}\mathbf{Y}] = \text{Tr}(\mathbf{K}_z)\mathbf{K}_x
\end{equation}%
Similarly, the covariance matrix over the columns of $\mathbf{Y}$ is 
\begin{equation}
\label{eq:column_cov}
    \text{Cov}_c[\mathbf{Y}] = E[\mathbf{Y}\mathbf{Y}^{\top}] = \text{Tr}(\mathbf{K}_x)\mathbf{K}_z
\end{equation}%
The proof for Eq.(\ref{eq:row_cov}) and Eq.(\ref{eq:column_cov}) can be found in \cite{ding2018matrix}. In Appendix \ref{sec: appendix_generative_model_interpretation}, we show that the maximisation of the log-posterior of the coefficient vector $\mathbf{a}$ leads to the objective in Eq.(\ref{eq: kronecker_regression}).

\section{A Smoothness Measure for Dependent Graph Signals}
\label{sec: new_smooothness}

From Eq.(\ref{eq: y_generative_marginal}), a noise-free version of $\mathbf{y}$ has a Kronecker product covariance structure $\mathbf{K}_x \otimes \mathbf{K}_z$. Recall that, for $i.i.d.$ Gaussian graph signals, the covariance is often modelled as the pseudo-inverse of the graph Laplacian matrix (see Eq.(\ref{eq: generative_Dong}) and \cite{Dong2016a, Lake2010a}), which is used for measuring the signal smoothness. Inspired by this observation, we define a notion of smoothness for the graph signals with two-side dependency using the Laplacian quadratic form as follows
\begin{equation}
\begin{split}
    \label{eq: smoothness_kernel}
    \mathbf{y}^{\top} \mathbf{L}_{\otimes} \mathbf{y} &= \mathbf{y}^{\top} (\mathbf{K}_x \otimes \mathbf{K}_z)^{\dagger} \mathbf{y} \\
    &= \text{vec}(\mathbf{Y})^{\top} \text{vec}(\mathbf{K}_z^{\dagger} \mathbf{Y} \mathbf{K}_x^{\dagger}) \\
    &= \text{Tr} (\mathbf{Y}^{\top} \mathbf{K}_z^{\dagger} \mathbf{Y} \mathbf{K}_x^{\dagger})
\end{split}
\end{equation}%
where $\mathbf{L}_{\otimes} = (\mathbf{K}_x \otimes \mathbf{K}_z)^{\dagger} = \mathbf{K}_x^{\dagger} \otimes \mathbf{K}_z^{\dagger}$ can be interpreted as a Laplacian-like operator with a Kronecker product structure. To see this more clearly, first, let us define $\mathcal{G}_x = \{\mathcal{V}_x, \mathcal{E}_x, \mathbf{L}_x\}$ as an undirected weighted column graph that represents the structure among column vectors, and correspondingly $\mathcal{G}_z = \{\mathcal{V}_z, \mathcal{E}_z, \mathbf{L}_z\}$ as a row graph that represents the structure among row vectors. Second, let us connect the kernel matrices $\mathbf{K}_x$ and $\mathbf{K}_z$ to the Laplacian matrices $\mathbf{L}_x$ and $\mathbf{L}_z$ by recognising that the former can be defined as functions of the latter as kernels on graphs \cite{smola2003kernels}, e.g.
\begin{equation}
    \nonumber
    \mathbf{K}_x = \mathbf{L}_x^{\dagger}, 
    \quad  \mathbf{K}_z = \mathbf{L}_z^{\dagger}. 
\end{equation}%
Therefore, we have $\mathbf{L}_{\otimes} = \mathbf{L}_x \otimes \mathbf{L}_z$, and we further show in Appendix \ref{sec: appendix_Kron_Laplacian} that $\mathbf{L}_{\otimes}$ is a Laplacian-like operator on which the notion of frequencies of $\mathbf{y}$ can be defined. \par

In practice, the observation-side dependency is often given or easy to obtain. For example, for graph signals with temporal Markovian dependency, $\mathbf{L}_z$ is often modelled as a path graph representing that the observation at time $\tau+1$ only depends on the observation at time $\tau$. By comparison, $\mathbf{L}_x$ is the primary variable of interest that is often estimated in the graph learning literature. Therefore, in this paper, we assume that $\mathbf{L}_z$ can be encoded in the observation-side information $\mathbf{z}$ via $\mathbf{K}_z$ such that $\mathbf{L}_z = \mathbf{K}_z^{\dagger}$. We simply denote $\mathbf{L}_x = \mathbf{L}$ from this section onwards and define a smoothness measure where we replace the kernel matrix $\mathbf{K}_x$ in Eq.(\ref{eq: smoothness_kernel}) with the Laplacian matrix $\mathbf{L}$
\begin{equation}
\begin{split}
    \label{eq: smoothness_multigraph}
    \mathbf{y}^{\top} \mathbf{L}_{\otimes} \mathbf{y} = \text{Tr} (\mathbf{Y}^{\top} \mathbf{K}_z^{\dagger} \mathbf{Y} \mathbf{L}).
\end{split}
\end{equation}%
This effectively disentangles the relationship among nodes (i.e. the graph to be learned) from the observation-side dependency in graph signals.

The smoothness term, in the vectorised form, can be viewed as a Laplacian regulariser which can be added to the problem of inferring a function $f$ that fits the graph signals in Eq.(\ref{eq: kronecker_regression}). Specifically, the graph signals in the smoothness term can be replaced with the estimates $\mathbf{\hat{Y}} = \mathbf{K}_z\mathbf{A}\mathbf{K}_x$ from the function $f$ such that
\begin{equation}
\begin{split}
    ||f||^2_{\mathcal{H}_{\mathcal{M}}} & = \langle f, \mathbf{L}_{\otimes} f \rangle_{\mathcal{H}_{\mathcal{M}}} \\
    & = \langle f, (\mathbf{L} \otimes \mathbf{K}_z^{\dagger}) f \rangle_{\mathcal{H}_{\mathcal{M}}} \\
    & = \text{vec}(\mathbf{\hat{Y}})^{\top}(\mathbf{L} \otimes \mathbf{K}_z^{\dagger})\text{vec}(\mathbf{\hat{Y}}) \\
    & =  \text{Tr}(\mathbf{A} \mathbf{K}_x \mathbf{L} \mathbf{K}_x \mathbf{A}^{\top} \mathbf{K}_z) \\ \label{eq: f_smoothness}
\end{split}
\end{equation}%
where $\mathcal{M}$ denotes a compact manifold\footnote{We refer the interested reader to \cite{belkin2005manifold, belkin2006manifold} for the theorem of manifold regularisation with the Laplace-Beltrami operator.}.
We will make use of the Laplacian regulariser in Eq.(\ref{eq: f_smoothness}) to derive the proposed graph learning models in the following section. \par

\section{Kernel Graph Learning}
\label{sec: methods_KGL}

\subsection{Learning Framework}
We propose a joint learning framework for inferring the function $f$ that fits the graph signals as in Eq.(\ref{eq: kronecker_regression}) 
as well as the underlying graph $\mathbf{L}$ to capture the relationship between the nodes as in Eq.(\ref{eq: f_smoothness}). This relationship is disentangled from the observation-side dependency of non-$i.i.d.$ graph signals with the notion of smoothness introduced in Section \ref{sec: new_smooothness}. We name this framework Kernel Graph Learning (\textbf{KGL}) which aims at solving the following problem:
\begin{equation}
\begin{split}
\label{eq:KGL_two_side}
    \min_{\mathbf{L} \in \mathcal{L}, \mathbf{A}} ~ & J(\mathbf{L}, \mathbf{A}) = ||\mathbf{Y} - \mathbf{K}_z \mathbf{A}  \mathbf{K}_x||_F^2  + \lambda \text{Tr}(\mathbf{K}_z \mathbf{A}  \mathbf{K}_x \mathbf{A}^{\top}) \\  & + \rho \text{Tr}(\mathbf{A} \mathbf{K}_x \mathbf{L} \mathbf{K}_x \mathbf{A}^{\top} \mathbf{K}_z) + \psi ||\mathbf{L}||^2_F \\
   \text{s.t.}  \quad & \text{Tr}(\mathbf{L}) = m
\end{split}
\end{equation}%
where $\mathcal{L} = \{\mathbf{L} | \mathbf{L}\mathbf{1} = 0,  \mathbf{L}_{jj'} = \mathbf{L}_{j'j} \leq 0, \forall j \neq j' \}$, and $||\cdot||_F$ denotes the Frobenius norm. 
The first two terms correspond to the functional learning part where the hyperparameter $\lambda > 0$ controls the complexity of the function $f$ for fitting $\mathbf{Y}$.
The last two terms and the constraints can be viewed as a graph learning model in Eq.(\ref{eq: smoothnessL}) with the fitted values of $\mathbf{Y}$ as input graph signals and the sum barrier as the graph regulariser. The hyperparameter $\rho > 0$ controls the relative importance between fitting the function and learning the graph, and $\psi > 0$ controls the distribution of edge weights. The trace constraint acts as a normalisation term such that the sum of learned edge weights equals the number of nodes. The model is compatible with constraints that enforce other properties on the learned graph, e.g. the log barrier introduced in Section \ref{sec:preliminary_GSP_GL}. This paper is mainly based on one of the choices for the constraints in order to maintain focus on the general framework. \par

\subsection{Optimisation: Alternating Minimisation}

We first recognise that Eq.(\ref{eq:KGL_two_side}) is a biconvex optimisation problem, i.e. convex w.r.t $\mathbf{A}$ while $\mathbf{L}$ is fixed and vice versa. This motivates an iterative block-coordinate descent algorithm that alternates between minimisation in $\mathbf{A}$ and $\mathbf{L}$ \cite{Tseng2001a, Gorski2007a}. In this section, we derive the update steps of $\mathbf{A}$ and $\mathbf{L}$ separately, propose the main algorithm in Algorithm \ref{alg:KGL_L}, and prove its convergence.

\subsubsection{Update of $\mathbf{A}$}
The update of coefficients $\mathbf{A}$ can be regarded as solving a Laplacian-regularised kernel regression \cite{belkin2006manifold}. Given $\mathbf{L}$, the optimisation problem of Eq.(\ref{eq:KGL_two_side}) becomes
\begin{equation}
\begin{split}
    \min_{\mathbf{A}} ~ & ||\mathbf{Y} - \mathbf{K}_z \mathbf{A}  \mathbf{K}_x||_F^2  + \lambda \text{Tr}(\mathbf{K}_z \mathbf{A}  \mathbf{K}_x \mathbf{A}^{\top}) \\ & + \rho \text{Tr}(\mathbf{A} \mathbf{K}_x \mathbf{L} \mathbf{K}_x \mathbf{A}^{\top} \mathbf{K}_z)
\end{split}
\end{equation}
and, after dropping constant terms,     
\begin{equation}
\label{eq:update_A_obj}
\begin{split}
    \min_{\mathbf{A}} ~ J_\mathbf{L}(\mathbf{A}) = &
    \text{Tr}(\mathbf{A}^{\top} \mathbf{K}_z^2 \mathbf{A}  \mathbf{K}_x^2) - 2 \text{Tr}(\mathbf{K}_x \mathbf{A}^{\top}  \mathbf{K}_z \mathbf{Y}) \\ & + \lambda \text{Tr}(\mathbf{K}_z \mathbf{A}  \mathbf{K}_x \mathbf{A}^{\top}) + \rho \text{Tr}(\mathbf{A} \mathbf{K}_x \mathbf{L} \mathbf{K}_x \mathbf{A}^{\top} \mathbf{K}_z).
\end{split}
\end{equation}%
Denote $\mathbf{a} = \text{vec}(\mathbf{A})$ and $\mathbf{y} = \text{vec}(\mathbf{Y})$, we obtain a dual-form for $J_\mathbf{L}(\mathbf{A})$ such that
\begin{equation}
\label{eq:update_A_obj_vec}
\begin{split}
     J_\mathbf{L}(\mathbf{a}) 
     = & \mathbf{a}^{\top}(\mathbf{K}_x^2 \otimes \mathbf{K}_z^2) \mathbf{a} - 2 \mathbf{a}^{\top} (\mathbf{K}_x \otimes \mathbf{K}_z)\mathbf{y} \\
     & + \lambda \mathbf{a}^{\top} (\mathbf{K}_x \otimes \mathbf{K}_z)\mathbf{a} + \rho \mathbf{a}^{\top} \big( (\mathbf{K}_x \mathbf{L} \mathbf{K}_x) \otimes \mathbf{K}_z \big)  \mathbf{a} \\
     = &  \mathbf{a}^{\top} \Big( \mathbf{K}^2 + \lambda \mathbf{K} + \rho \mathbf{S} \otimes \mathbf{K}_z \Big) \mathbf{a} - 2 \mathbf{a}^{\top} \mathbf{K} \mathbf{y}
\end{split}
\end{equation}%
where $\mathbf{K} = \mathbf{K}_x \otimes \mathbf{K}_z$, and $\mathbf{S} = \mathbf{K}_x \mathbf{L} \mathbf{K}_x$ for simplicity. We prove in Appendix \ref{sec: appendix_psd} that $\mathbf{K}^2 + \lambda \mathbf{K} + \rho \mathbf{S} \otimes \mathbf{K}_z $ is positive semi-definite thus Eq.(\ref{eq:update_A_obj_vec}) is an unconstrained quadratic programme. The gradient of $J_\mathbf{L}(\mathbf{a})$ w.r.t. $\mathbf{a}$ is 
\begin{equation}
\label{eq:vecA_graident_K_not_invertible}
    \nabla J_\mathbf{L}(\mathbf{a}) = \Big( \mathbf{K}^2 + \lambda \mathbf{K} + \rho \mathbf{S} \otimes \mathbf{K}_z \Big) \mathbf{a} - \mathbf{K} \mathbf{y}.
\end{equation}%

Strictly speaking, the matrix $\mathbf{K}^2 + \lambda \mathbf{K} + \rho \mathbf{S} \otimes \mathbf{K}_z$ may contain zero eigenvalues which makes it not invertible. However, a majority of popular kernel functions for $\mathbf{K}_x$ and $\mathbf{K}_z$ are positive-definite, e.g. the RBF kernel. Since Kronecker product preserves positive definiteness, $\mathbf{K}$ and hence the whole matrix is positive-definite and invertible. 
Setting $\nabla J_\mathbf{L}(\mathbf{a}) = \mathbf{0}$ and cancelling out $\mathbf{K}$, we have:
\begin{equation}
\label{eq: update_a_gradient_cancel_K}
\begin{split}
    & \Big( \mathbf{K}^2 + \lambda \mathbf{K} + \rho \mathbf{S} \otimes \mathbf{K}_z \Big) \mathbf{a} - \mathbf{K} \mathbf{y} = \mathbf{0} \\
    \implies & \Big( \mathbf{K}^2 + \lambda \mathbf{K} + \rho \mathbf{K}(\mathbf{L}\mathbf{K}_x \otimes \mathbf{I}_n) \Big) \mathbf{a} - \mathbf{K} \mathbf{y} = \mathbf{0} \\
    \implies & \Big( \mathbf{K} + \lambda \mathbf{I}_{mn} + \rho \mathbf{K}(\mathbf{L}\mathbf{K}_x \otimes \mathbf{I}_n) \Big) \mathbf{a} = \mathbf{y}
\end{split}
\end{equation}
Denote $\mathbf{H} = \mathbf{K} + \lambda \mathbf{I}_{mn} + \rho \mathbf{K}(\mathbf{L}\mathbf{K}_x \otimes \mathbf{I}_n)$, where $\mathbf{H}$ has a dimension of $nm \times nm$. We can therefore obtain a closed-form solution such that
\begin{equation}
    \mathbf{a} = \mathbf{H}^{-1}\mathbf{y}
\end{equation}%
where the inverse of $\mathbf{H}$ requires $\mathcal{O}(n^3m^3)$. 

To further reduce the complexity, we make use of the Kronecker structure and matrix tricks. We first recognise $\mathbf{H}$ as
\begin{equation}
\nonumber
\mathbf{H} = \left(\left(\rho \mathbf{L} +\eta \mathbf{K}_{x}^{-1}\right)\oplus \mathbf{K}_{z}\right)\left(\mathbf{K}_{x}\otimes \mathbf{I}_{n}\right)
\end{equation}%
where $\oplus$ is the Kronecker sum. With the eigendecomposition $\mathbf{K}_{x}= \mathbf{Q}_{x} \boldsymbol{\Lambda}_{x}\mathbf{Q}_{x}^{\top}$, $\mathbf{K}_{z}=\mathbf{Q}_{z}\boldsymbol{\Lambda}_{z}\mathbf{Q}_{z}^{\top}$
and $\rho \mathbf{L} +\eta \mathbf{K}_{x}^{-1}=\mathbf{U}_{x}\mathbf{D}_{x}\mathbf{U}_{x}^{\top}$, we have 
\begin{equation}
    \mathbf{H} = \left(\mathbf{U}_{x}\otimes \mathbf{Q}_{z}\right)\left(\mathbf{D}_{x}\oplus\boldsymbol{\Lambda}_{z}\right)\left(\mathbf{U}_{x}^{\top}\otimes \mathbf{Q}_{z}^{\top}\right)\left(\mathbf{Q}_{x}\boldsymbol{\Lambda}_{x}\mathbf{Q}_{x}^{\top}\otimes \mathbf{I}_{n}\right).
\end{equation}
Here, $\mathbf{D}_{x}\oplus\boldsymbol{\Lambda}_{z}$ is an $mn\times mn$ diagonal matrix with entries being all the pairwise sums of eigenvalues in $\mathbf{D}_{x}$ and $\boldsymbol{\Lambda}_{z}$. Inversion of that matrix is thus $O(mn)$. We can now obtain cheap
inversion with
\begin{equation}
\begin{split}
     \mathbf{H}^{-1}\mathbf{y} = & \left(\mathbf{Q}_{x}\boldsymbol{\Lambda}_{x}^{-1}\mathbf{Q}_{x}^{\top}\otimes \mathbf{I}_{n}\right)\left(\mathbf{U}_{x}\otimes \mathbf{Q}_{z}\right)\left(\mathbf{D}_{x}\oplus\boldsymbol{\Lambda}_{z}\right)^{-1}\\
     &\cdot\text{vec}\left(\mathbf{Q}_{z}^{\top}\mathbf{Y}\mathbf{U}_{x}\right).
\end{split}
\end{equation}
The operation $\left(\mathbf{D}_{x}\oplus\boldsymbol{\Lambda}_{z}\right)^{-1}\text{vec}\left(\mathbf{Q}_{z}^{\top}\mathbf{YU}_{x}\right)$
is simply rescaling each term in the $mn$-vector with the corresponding
diagonal entry of $\left(\mathbf{D}_{x}\oplus\boldsymbol{\Lambda}_{z}\right)^{-1}$. If
we denote $\mathbf{d}_{x}$ and $\mathbf{d}_{z}$ column vectors containing the
diagonal entries this can be expressed as $\text{vec}\left(\mathbf{B}\right)$ with
\[
\mathbf{B}=\left(\mathbf{1}_{n}\mathbf{d}_{x}^{\top}+\mathbf{d}_{z}\mathbf{1}_{m}\right)^{\circ-1}\circ\left(\mathbf{Q}_{z}^{\top}\mathbf{YU}_{x}\right),
\]
where $\circ$ denotes the Hadamard product and $(\cdot)^{\circ-1}$ denotes entrywise
inversion. Remaining operations are now direct and give the closed-form solution as
\begin{eqnarray}
\mathbf{A} & = & \mathbf{Q}_{z}\mathbf{BU}_{x}^{\top}\mathbf{Q}_{x}\boldsymbol{\Lambda}_{x}^{-1}\mathbf{Q}_{x}^{\top}\nonumber \\
 & = & \mathbf{Q}_{z}\mathbf{BU}_{x}^{\top}\mathbf{K}_{x}^{-1}\nonumber \\
 & = & \mathbf{Q}_{z}\left[\left(\mathbf{1}_{n}\mathbf{d}_{x}^{\top}+\boldsymbol{\lambda}_{z}\mathbf{1}_{m}\right)^{\circ-1}\circ\left(\mathbf{Q}_{z}^{\top}\mathbf{YU}_{x}\right)\right]\mathbf{U}_{x}^{\top}\mathbf{K}_{x}^{-1}. \nonumber \\
 \label{eq:update_a_vec_closed_kron}
\end{eqnarray}
Notice that this solution requires only matrix multiplications and inversions
and eigendecompositions on $m\times m$ or $n\times n$ matrices,
giving an overall computational cost of $\mathcal{O}\left(n^{3}+m^{3}+nm^{2}+n^{2}m\right)$.

When $\mathbf{K}_x$ and $\mathbf{K}_z$ are not invertible,  we suggest using the gradient descent to avoid the inverse of a large matrix of dimension $nm \times nm$. The update step using Eq.(\ref{eq:vecA_graident_K_not_invertible}) is:
\begin{equation}
\label{eq: KGL_update_vec_a_GD}
    \mathbf{a}^{(\tau + 1)} = \mathbf{a}^{(\tau)}  - \gamma \nabla J_\mathbf{L}(\mathbf{a}^{(\tau)})
\end{equation}%
where $\gamma > 0$ is the learning rate

\subsubsection{Update of $\mathbf{L}$}
Given $\mathbf{A}$, 
the optimisation problem of Eq.(\ref{eq:KGL_two_side}) becomes
\begin{equation}
\label{eq: KGL_update_L}
\begin{split}
    \min_{\mathbf{L} \in \mathcal{L}} ~ & J_{\mathbf{A}}(\mathbf{L})= \rho \text{Tr}(\mathbf{A} \mathbf{K}_x \mathbf{L} \mathbf{K}_x \mathbf{A}^{\top} \mathbf{K}_z) + \psi ||\mathbf{L}||^2_F \\
   \text{s.t} \quad & \text{Tr}(\mathbf{L}) = m
\end{split}
\end{equation}%
which is a constrained quadratic programme w.r.t. $\mathbf{L}$. By taking $\mathbf{P} = \mathbf{K}_z^{1/2}\mathbf{A}\mathbf{K}_x$, the problem fits in the learning framework in Eq.(\ref{eq: smoothnessL}). We use the package CVXPY \cite{diamond2016cvxpy} to solve this problem. \par

\begin{algorithm}[h!]
\caption{Kernel Graph Learning (\textbf{KGL})}
\label{alg:KGL_L}
 \begin{algorithmic}[1]
 \renewcommand{\algorithmicrequire}{\textbf{Input:}}
 \renewcommand{\algorithmicensure}{\textbf{Output:}}
 \REQUIRE Observation $\mathbf{Y}$, node-side kernel matrix $\mathbf{K}_x$, observation-side kernel matrix $\mathbf{K}_z$, hyper-parameters $\rho$, $\lambda$ and $\psi$, tolerance level $ \epsilon$.
\STATE \textbf{Initialisation: $t = 0, \mathbf{A} = \mathbf{0} \in \mathbb{R}^{n \times m}$} \par
\WHILE {$|\mathbf{L}^{(t)} - \mathbf{L}^{(t-1)}| < \epsilon$ and $|\mathbf{A}^{(t)} - \mathbf{A}^{(t-1)}| < \epsilon$}
  \STATE Update $\mathbf{L}^{(t)} = \arg\min J_{\mathbf{A}^{(t-1)}}(\mathbf{L})$ by solving Eq.(\ref{eq: KGL_update_L})
  \STATE Update $\mathbf{A}^{(t)} = \arg\min J_{\mathbf{L}^{(t)}}(\mathbf{A})$ by \par
  (i) using the closed-form solution in Eq.(\ref{eq:update_a_vec_closed_kron}), \textbf{if} $\mathbf{K}_x$ and $\mathbf{K}_z$ are invertible; or \par
  (ii) updating $\text{vec}(\mathbf{A}^{(t)})$ with gradient descent in Eq.(\ref{eq: KGL_update_vec_a_GD}), \textbf{otherwise}
  \STATE $t = t + 1$
\ENDWHILE
\RETURN $\mathbf{L}^{(t)}$, $\mathbf{A}^{(t)}$
\end{algorithmic} 
\end{algorithm}

The overall KGL framework is presented in Algorithm \ref{alg:KGL_L}.
The convergence for each update step of $\mathbf{A}^{(t)}$ and $\mathbf{L}^{(t)}$ is guaranteed in solving the respective convex optimisation of Eq.(\ref{eq:update_A_obj}) and Eq.(\ref{eq: KGL_update_L}). It should be noted that the step size $\gamma$ in Eq.(\ref{eq: KGL_update_vec_a_GD}) needs to be set appropriately for the gradient descent to converge. We suggest a $\gamma \leq 10^{-4}$ from empirical results. We now prove the following lemma. 

\begin{lemma}
The sequence $\{J(\mathbf{L}^{(t)}, \mathbf{A}^{(t)})\}$ generated by Algorithm \ref{alg:KGL_L} converges monotonically and the solution obtained by Algorithm \ref{alg:KGL_L} is a stationary point of  Eq.(\ref{eq:KGL_two_side}).
\end{lemma}

\begin{proof}
We follow the convergence results of the alternate convex search in \cite{Gorski2007a} and that of a more general cyclic block-coordinate descent algorithm in \cite{Tseng2001a}. By recognising Eq.(\ref{eq:update_A_obj}) and Eq.(\ref{eq: KGL_update_L}) are quadratic programmes (with Lemma \ref{lemma: optiA_matrix_psd} in Appendix \ref{sec: appendix_psd}), the problem of Eq.(\ref{eq:KGL_two_side}) is a bi-convex problem with all the terms differentiable and the function $J(\mathbf{L}, \mathbf{A})$ continuous and bounded from below. Theorem 4.5 in \cite{Gorski2007a} states that the sequence $\{J(\mathbf{L}^{(t)}, \mathbf{A}^{(t)})\}$ generated by Algorithm \ref{alg:KGL_L} converges monotonically. Theorem 4.1 in \cite{Tseng2001a} states that the sequence $\{\mathbf{L}^{(t)}\}$ and $\{\mathbf{A}^{(t)}\}$ generated by Algorithm \ref{alg:KGL_L} are defined and bounded. Furthermore, according to Theorem 5.1 in \cite{Tseng2001a}, every cluster point $\{\mathbf{L}^{(t)}, \mathbf{A}^{(t)} \}$ is a coordinatewise minimum point of $J$ hence the solution is a stationary point of Eq.(\ref{eq:KGL_two_side}).

\end{proof}

Our empirical results suggest that after only 10 iterations or less, the sequence $\{\mathbf{L}^{(t)}, \mathbf{A}^{(t)} \}$ does not change more than the tolerance level. The computational complexity of \textbf{KGL} in Algorithm \ref{alg:KGL_L} is dominated by the step of updating $\mathbf{A}$. It requires $\mathcal{O}(n^3+m^3+nm^2+n^2m)$ to compute the closed-form solution of $\mathbf{A}$ or $\mathcal{O}(n^3m^3)$ to compute the gradient in Eq.(\ref{eq:vecA_graident_K_not_invertible}) if the closed-form solution is not applied when $\mathbf{K}_z$
and $\mathbf{K}_x$ not invertible. Updating $\mathbf{L}$ requires $\mathcal{O}(m^2)$. Overall, for $T$ iterations that guarantee the convergence of Algorithm \ref{alg:KGL_L}, it requires either $\mathcal{O}(T(n^3+m^3+nm^2+n^2m))$ or $\mathcal{O}(T(n^3m^3))$ operations, depending on whether $\mathbf{K}_z$
and $\mathbf{K}_x$ are chosen to be invertible. We note that one could readily appeal to large-scale kernel approximation methods for further reduction of computational and storage complexity of the \textbf{KGL} framework, and hence broaden its applicability to larger datasets. There are two main approaches to large-scale kernel approximations and both can be applied to \textbf{KGL}. The former focuses on kernel matrix approximations using methods such as Nystr\"om sampling \cite{kumar2012sampling}, while the latter deals with the approximation of the kernel function itself, using methods such as Random Fourier Features \cite{francis2020major}. Nonetheless, this paper focuses on the modelling perspective, and we will leave the algorithmic improvement as a future direction.


\subsection{Special Cases of Kernel Graph Learning}

\paragraph{Independent observations} It is often assumed that graph signals are $i.i.d.$ hence there exists no dependency along the observation side. This is equivalent to setting $\mathbf{K}_z = \mathbf{I}_n$ in our framework. We refer to this special case of \textbf{KGL} as Node-side Kernel Graph Learning (\textbf{KGL-N}):
\begin{equation}
\label{eq:KGL_one_side_node}
\begin{split}
    \min_{\mathbf{L} \in \mathcal{L}, \mathbf{A}} ~ & ||\mathbf{Y} - \mathbf{A}  \mathbf{K}_x||_F^2  + \lambda \text{Tr}(\mathbf{A}  \mathbf{K}_x \mathbf{A}^{\top}) \\ & + \rho \text{Tr}(\mathbf{A} \mathbf{K}_x \mathbf{L} \mathbf{K}_x \mathbf{A}^{\top}) + \psi ||\mathbf{L}||^2_F \\
   \text{s.t} \quad & \text{Tr}(\mathbf{L}) = m
\end{split}
\end{equation}%

\paragraph{No node-side information} It also may be the case that no node-side information is available for the problem at hand. In this case, we can simply set $\mathbf{K}_x = \mathbf{I}_m$ in \textbf{KGL}, leading to to Observation-side Kernel Graph Learning (\textbf{KGL-O}):
\begin{equation}
\label{eq:KGL_one_side_obs}
\begin{split}
    \min_{\mathbf{L} \in \mathcal{L}, \mathbf{A}} ~ & ||\mathbf{Y} - \mathbf{K}_z \mathbf{A}||_F^2  + \lambda \text{Tr}(\mathbf{A}^{\top} \mathbf{K}_z \mathbf{A}) \\ & + \rho  \text{Tr}(\mathbf{A} \mathbf{L} \mathbf{A}^{\top} \mathbf{K}_z) + \psi ||\mathbf{L}||^2_F \\
   \text{s.t} \quad & \text{Tr}(\mathbf{L}) = m
\end{split}
\end{equation}%

In the cases of one-side \textbf{KGL}, the optimisation again follows the alternating minimisation, i.e. to solve for \textbf{KGL-N} or \textbf{KGL-O}), one can simply set $\mathbf{K}_z = \mathbf{I}_n$ or $\mathbf{K}_x = \mathbf{I}_m$ in Algorithm \ref{alg:KGL_L}. It should be noted, however, that the update step of $\mathbf{A}$ requires less computational cost when either $\mathbf{K}_z = \mathbf{I}_n$ or $\mathbf{K}_x = \mathbf{I}_m$. Indeed, the objective function of $\textbf{KGL-N}$ can be decomposed into the sum according to $n$ functions
such that
\begin{equation}
\begin{split}
    J_{\mathbf{L}}(\{\mathbf{a}_i\}_{i=1}^n) =  \sum_{i=1}^n  \Big( ||\mathbf{y}_i - \mathbf{K}_x \mathbf{a}_i||^2_2 + \mathbf{a}_i^{\top} (\lambda \mathbf{K}_x +  \rho  \mathbf{S}) \mathbf{a}_i \Big)
\end{split}    
\end{equation}%
where $\mathbf{A} = [\mathbf{a}_1, \mathbf{a}_2, \dots, \mathbf{a}_n]^{\top}$ and $\mathbf{S} = \mathbf{K}_x \mathbf{L} \mathbf{K}_x$. Consequently, the update step can be parallelised.
\subsection{Learning with Missing Observations}
\label{sec: missing}
By modifying the least-squares loss in \textbf{KGL}, we propose an extension to jointly learn the underlying graph and function from graph-structured data with missing values. 
We encode the positions of missing values with a mask matrix $\mathbf{M} \in \mathbb{R}^{n \times m}$ such that $\mathbf{M}_{ij} = 0$ if $\mathbf{Y}_{ij}$ is missing, and $\mathbf{M}_{ij} = 1$ otherwise. Now, we only need to minimise the least-squares loss over observed $\mathbf{Y}_{ij}$ in the functional learning part, which leads to the formulation:
\begin{equation}
\label{eq:KGL_two_side_missing_obj}
\begin{split}
    \min_{\mathbf{L} \in \mathcal{L}, \mathbf{A}} ~ & ||\mathbf{M} \circ(\mathbf{Y} - \mathbf{K}_z \mathbf{A}  \mathbf{K}_x)||_F^2  + \lambda \text{Tr}(\mathbf{K}_{\mathcal{Z}} \mathbf{A}  \mathbf{K}_x \mathbf{A}^{\top}) \\ & + \rho \text{Tr}(\mathbf{A} \mathbf{K}_x \mathbf{L} \mathbf{K}_x \mathbf{A}^{\top} \mathbf{K}_z) + \psi ||\mathbf{L}||^2_F \\
   \text{s.t} \quad & \text{Tr}(\mathbf{L}) = m.
\end{split}
\end{equation}%
This formulation also applies to one-side kernel graph learning, i.e. \textbf{KGL-N} or \textbf{KGL-O}, with $\mathbf{K}_z = \mathbf{I}_n$ or $\mathbf{K}_x = \mathbf{I}_m$.

The optimisation problem in Eq.(\ref{eq:KGL_two_side_missing_obj}) is a bi-convex problem and alternating minimisation can still be applied. The update step of $\mathbf{L}$ remains the same as in Eq.(\ref{eq: KGL_update_L}), but the gradient in the update step of $\mathbf{a} = \text{vec}(\mathbf{A})$ (Step 4. in Algorithm \ref{alg:KGL_L}) becomes 
\begin{equation}
\label{eq: SKGL_gradient}
    \nabla J_{\mathbf{L}}(\mathbf{a}) = \Big( \mathbf{K}  \text{diag}(\mathbf{m})  \mathbf{K} + \lambda \mathbf{K} +  \rho \mathbf{S} \otimes \mathbf{K}_z \Big)\mathbf{a} - \mathbf{K} \text{vec}(\mathbf{M} \circ \mathbf{Y}) 
\end{equation}%
where $\mathbf{m} = \text{vec}(\mathbf{M})$. The detailed derivation of the gradient is provided in Appendix \ref{sec: appendix_missing_optimisation}. We further assume $\mathbf{K}$ is invertible, which is a mild assumption as we have many choices of kernel functions for $\mathbf{K}_x$ and $\mathbf{K}_z$ to be invertible. Also noting $\mathbf{S} \otimes \mathbf{K}_z = \mathbf{K}(\mathbf{LK}_x \otimes \mathbf{I}_n)$, the gradient becomes 
\begin{equation}
\label{eq: SKGL_gradient_noK}
    \nabla J_{\mathbf{L}}(\mathbf{a}) = \Big( \text{diag}(\mathbf{m})  \mathbf{K} + \lambda \mathbf{I}_{nm} +  \rho (\mathbf{LK}_x \otimes \mathbf{I}_n) \Big)\mathbf{a} -  \text{vec}(\mathbf{M} \circ \mathbf{Y}). 
\end{equation}%
One can either derive a close-form solution or use gradient descent based on Eq.(\ref{eq: SKGL_gradient_noK}). 

\section{Synthetic Experiments}
\label{sec: syn_exp}

\subsection{General Settings}
\label{sec: syn_exp_general_settings}
\paragraph{Groundtruth Graphs} Random graphs of $m$ nodes are drawn from the Erdös-Rényi (ER), Barabási-Albert (BA) and stochastic block model (SBM) as groundtruth, which are denoted as $\mathcal{G}_{\text{ER}}$, $\mathcal{G}_{\text{BA}}$ and $\mathcal{G}_{\text{SBM}}$, respectively. The parameters of each network model are chosen to yield an edge density of 0.3. The edge weights are randomly drawn from a uniform distribution $\mathbf{W}_{ij} \sim \mathcal{U}(0,1)$. The weighted adjacency matrix is set as $\mathbf{W} = (\mathbf{W} + \mathbf{W}^{\top})/2$ for symmetry and normalised such that the sum of edge weights is equal to $m$ for ease in comparison. The graph Laplacian $\mathbf{L}$ is calculated from $\mathbf{L} = \text{diag}(\mathbf{W1}) - \mathbf{W}$. \par

\paragraph{Groundtruth Data} We generate mild noisy data $\mathbf{Y} \in \mathbb{R}^{n \times m}$ from $\mathbf{Y} = \mathbf{K}_z\mathbf{A} \mathbf{K}_x + \mathbf{E}$, where $\mathbf{a} = \text{vec}(\mathbf{A})$ is drawn from $\mathbf{a} \sim \mathcal{N}(\mathbf{0}_{mn}, \mathbf{K}_{x}^{\dagger} \otimes \mathbf{K}_{z}^{\dagger})$ according to Eq.(\ref{eq: y_generative_a_prior}). Every entry of the noise matrix $\mathbf{E}$ is an $i.i.d.$ sample from $\mathbf{E}_{ij} \sim \mathcal{N}(0, \sigma_\epsilon^2)$. To test the proposed model against different levels of noises, we vary the value of $\sigma_\epsilon^2$ in Section \ref{sec:syn_exp_noisy}. For all other synthetic experiments, we add a mild-level noise with $\sigma_\epsilon^2 = 0.01$. We choose $\mathbf{K}_{x} = (\mathbf{I} + \alpha \mathbf{L})^{-1}$, as it is a popular method to generate smooth signals in related work \cite{kalofolias2016learn, venkitaraman2019predicting}. We consider both dependence and independence along the observation side:
\begin{itemize}
    \item \textbf{Independent Data}: $\mathbf{K}_{z} = \mathbf{I}_n$;
    \item \textbf{Dependent Data}: $\mathbf{K}_{z}$ is obtained from an RBF kernel evaluated on synthetic observation-side information $\mathbf{z} = [0, 1, 2, \dots, n-1]^{\top}$, which can be interpreted as the time-stamps of a discrete-time Markov chain. The bandwidth parameter is chosen according to the median heuristic \cite{garreau2017large}. 
\end{itemize}

\paragraph{Model Candidates} To have a fair comparison, the models are divided into two groups. The first group contains the baseline models that cannot deal with observation-side dependence:
\begin{itemize}
    \item \textbf{GL} (Eq.(14) in \cite{kalofolias2016learn}): the GSP graph learning model in Eq.(\ref{eq: smoothnessL}) with $\Omega(\mathbf{L}) = ||\mathbf{L}||^2_F$.
    \item \textbf{GL-2step} (Eq.(16) in \cite{Dong2016a}): a two-step GSP graph learning framework with an identity mapping as denoising function. From a modelling perspective, Eq.(13) in \cite{berger2020efficient} proposed a similar model with more constraints on edge weights and a different optimisation algorithm. We treat them as the same kind of techniques.
    \item \textbf{KGL-N} (proposed model in Eq. (\ref{eq:KGL_one_side_node})): $\mathbf{K}_z = \mathbf{I}_n$ in \textbf{KGL}.
\end{itemize}%
The second group is examined with observation-side dependent data: 
\begin{itemize}
    \item \textbf{KGL-Agnostic} (Eq.(18) in \cite{venkitaraman2019predicting}):  As discussed in Section \ref{sec: related_work}, the joint learning model in \cite{venkitaraman2019predicting} considered the observation-side kernel, but did not use the observation-side dependence in learning the graph. We denote their model as \textbf{KGL-Agnostic} with our notations: 
    \begin{equation}
    \label{eq:baseline_KGL_ObsDep}
    \begin{split}
        \min_{\mathbf{L} \in \mathcal{L}, \mathbf{A}} ~ & ||\mathbf{Y} - \mathbf{K}_z \mathbf{A}||_F^2  + \lambda \text{Tr}(\mathbf{A}^{\top} \mathbf{K}_z \mathbf{A}) \\ & + \rho  \text{Tr}(\mathbf{K}_{z} \mathbf{A} \mathbf{L} \mathbf{A}^{\top} \mathbf{K}_z) + \psi ||\mathbf{L}||^2_F
    \end{split}
    \end{equation}%

    \item \textbf{KGL} (proposed model in Eq. (\ref{eq:KGL_two_side})): the main learning framework.
    \item \textbf{KGL-O} (proposed model in Eq.(\ref{eq:KGL_one_side_obs})): To have a fair comparison to \textbf{\textbf{KGL-Agnostic}}, we also assume the graph is agnostic to the model (i.e. $\mathbf{K}_{x} = \mathbf{I}_m$ as model input).
\end{itemize}%
For each method, we determine the hyperparameters via a grid search, and report the highest performance achieved by the best set of hyperparameters.

\paragraph{Evaluation Metrics} Average precision score (APS) and normalised sum of squared errors ($\text{SSE}_\mathcal{G}$) are used to evaluate the graph estimates, and out-of-sample mean squared error ($\text{MSE}_y$) is used to evaluate the estimated entries of graph-structured data matrix that were not observed (or missing). 
The APS is defined in a binary classification scenario for graph structure recovery, which automatically varies the threshold of weights above which the edges are declared as learned edges. An APS score of 1 indicates that the algorithm can precisely detect the ground-truth edges and non-edges. The $\text{SSE}_\mathcal{G}$ is defined over learned adjacency matrix $\hat{\mathbf{W}}$ and the groundtruth adjacency matrix $\mathbf{W}_0$ as 
\begin{equation}
\nonumber
\text{SSE}_\mathcal{G} = \frac{||\hat{\mathbf{W}} - \mathbf{W}_0||^2_F}{||\mathbf{W}_0||^2_F}.
\end{equation}%
The out-of-sample $\text{MSE}_y$ of data matrix is defined with a mask matrix $\mathbf{M}$ (same as in Eq.(\ref{eq:KGL_two_side_missing_obj})), where $\mathbf{M}_{ij} = 0$ indicates $\mathbf{Y}_{ij}$ is a missing entry:
\begin{equation}
\nonumber
\text{Out-of-sample MSE}_y = \frac{||(\mathbf{1} \mathbf{1}^{\top}- \mathbf{M}) \circ (\hat{\mathbf{Y}} - \mathbf{Y})||^2_F}{|| \mathbf{1}\mathbf{1}^{\top}-\mathbf{M}||^2_F}
\end{equation}%
where $\hat{\mathbf{Y}} = \mathbf{K}_{z} \mathbf{A} \mathbf{K}_{x}$ and $\mathbf{A}$ is obtained from model estimates. Similarly, we are interested in the training $\text{MSE}_y$ for analysing overfitting: 
\begin{equation}
\nonumber
\text{Training MSE}_y = \frac{||\mathbf{M} \circ (\hat{\mathbf{Y}} - \mathbf{Y})||^2_F}{|| \mathbf{M}||^2_F}.
\end{equation}%

\subsection{Learning a Graph from Noisy Data}
\label{sec:syn_exp_noisy}

In order to evaluate the performance of the proposed model in learning a graph from noisy data, we add noise to the groundtruth data such that $\mathbf{Y} = \mathbf{K}_z\mathbf{A}\mathbf{K}_x + \mathbf{E}$, where every entry of the noise matrix $\mathbf{E}$ is an $i.i.d.$ sample from  $\mathbf{E}_{ij} \sim \mathcal{N}(0, \sigma_\epsilon^2)$. We vary the noise level $\sigma_\epsilon^2$ from 0 to 2 against which we plot the evaluation metrics in Figure \ref{fig:exp_noisy}. Under the same settings of the noise level, random graph and model candidate, we repeat the experiment for 10 times and report the mean (the solid curves) as well as the 5th and 95th percentile (the error bars) of the evaluation metrics. \par
  
From Figure \ref{fig:exp_noisy}, Figure \ref{fig:exp_noisy_ER_appendix} and Figure \ref{fig:exp_noisy_BA_appendix} (the latter two in Appendix \ref{sec: appendix_figures_noisy}), the proposed models outperform the baseline models in terms of all evaluation metrics. Specifically, for $\mathcal{G}_{\text{SBM}}$, when the data are independent, the performance of \textbf{KGL-N} drops slowly as the noise level increases, while that of \textbf{GL} and \textbf{GL-2step} drops quickly when noise level goes above 0.5. It is worth mentioning that the curves of two almost overlap in terms of APS. This indicates the identity mapping in \textbf{GL-2step} as denoising function does not help much in recovering graph structure, although it yields slightly smaller SSE$_\mathcal{G}$ than \textbf{GL} with the noise level greater then 0.75. \par

\begin{figure}
\centering
\begin{subfigure}[t]{0.25\textwidth}
    \centering
    \includegraphics[width = 1\linewidth]{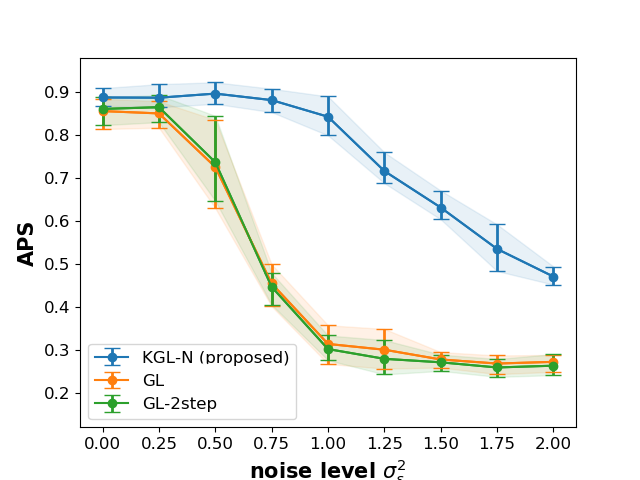}
    \subcaption{$\mathcal{G}_{\text{SBM}}$, independent data}
    \label{fig:exp_noisy_SBM_APS}
\end{subfigure}%
\begin{subfigure}[t]{0.25\textwidth}
    \centering
    \includegraphics[width = 1\linewidth]{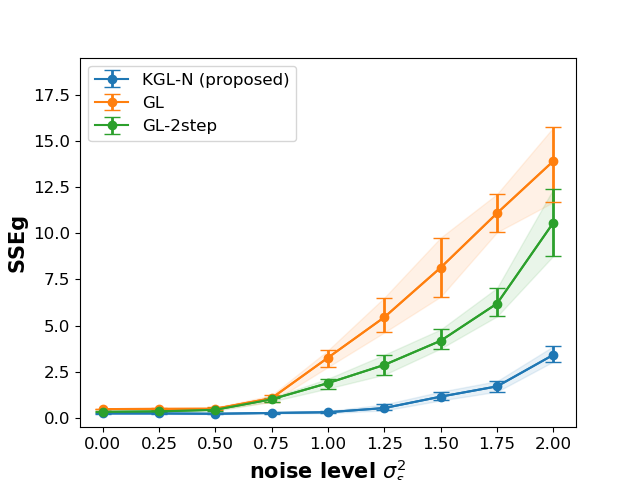}
     \subcaption{$\mathcal{G}_{\text{SBM}}$, independent data}
    \label{fig:exp_noisy_SBM_SSEg}
\end{subfigure}
\begin{subfigure}[t]{0.25\textwidth}
    \centering
    \includegraphics[width = 1\linewidth]{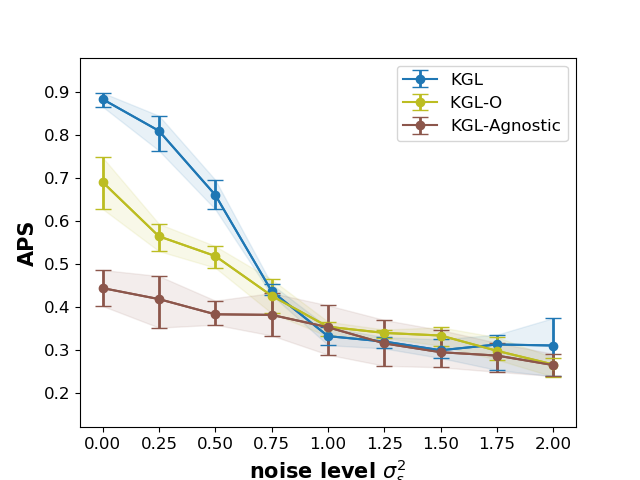}
    \subcaption{$\mathcal{G}_{\text{SBM}}$, dependent data}
    \label{fig:exp_noisy_SBM_APS_2side}
\end{subfigure}%
\begin{subfigure}[t]{0.25\textwidth}
    \centering
    \includegraphics[width = 1\linewidth]{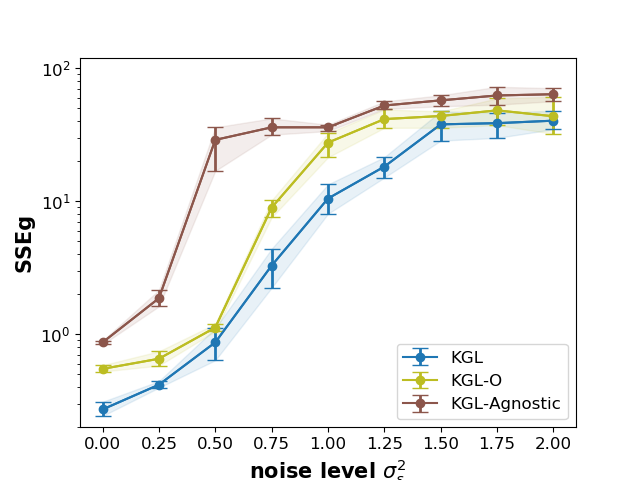}
     \subcaption{$\mathcal{G}_{\text{SBM}}$, dependent data}
    \label{fig:exp_noisy_SBM_SSEg_2side}
\end{subfigure}
\caption{The performance of recovering groundtruth graphs $\mathcal{G}_{\text{SBM}}$ in terms of APS and SSE$_\mathcal{G}$ from independent data (1st row) and dependent data (2nd row) with different noise levels.} 
\label{fig:exp_noisy}
\end{figure}

For the dependent data, the proposed model $\textbf{KGL}$ achieves a high performance when the noise level is less than 0.75. Even without the node-side information $\mathbf{K}_x$ as model input in \textbf{KGL} (note that the groundtruth data are generated in a consistent way in \cite{venkitaraman2019predicting} proposing \textbf{KGL-agnostic}), the proposed method ($\textbf{KGL-O}$) can still learn a meaningful graph with slightly worse performance compared to $\textbf{KGL}$. By contrast, \textbf{KGL-agnostic} cannot recover the groundtruth graph to a satisfying level even with little noise ($\sigma_\epsilon^2 = 0)$, as its smoothness term does not capture the dependence structure on the observation side. \par

From Figure \ref{fig:exp_noisy_ER_appendix} and Figure \ref{fig:exp_noisy_BA_appendix}, we see that $\mathcal{G}_{\text{BA}}$ is slightly more difficult to recover from data, but an improvement can nevertheless be seen in the proposed models from the baselines when the noise level is low. \par

\subsection{Learning a Graph from Missing Data}
\label{sec: syn_exp_missing_graph}

\begin{figure}
\centering
 \begin{subfigure}[t]{0.25\textwidth}
    \centering
    \includegraphics[width = 1\linewidth]{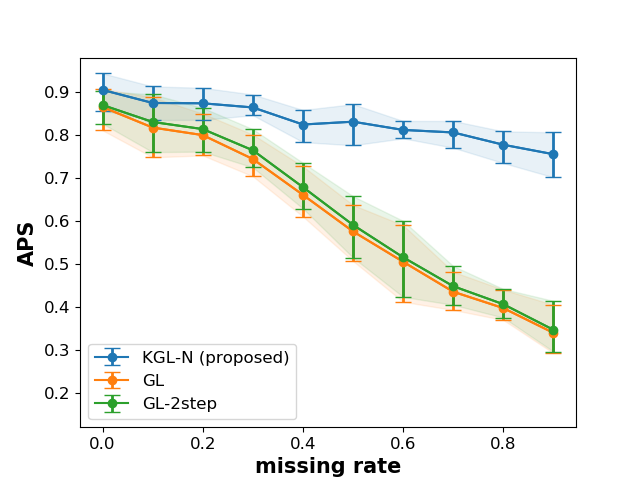}
    \subcaption{$\mathcal{G}_{\text{SBM}}$, independent data}
    \label{fig:exp_missing_SBM_APS}
\end{subfigure}%
\begin{subfigure}[t]{0.25\textwidth}
    \centering
    \includegraphics[width = 1\linewidth]{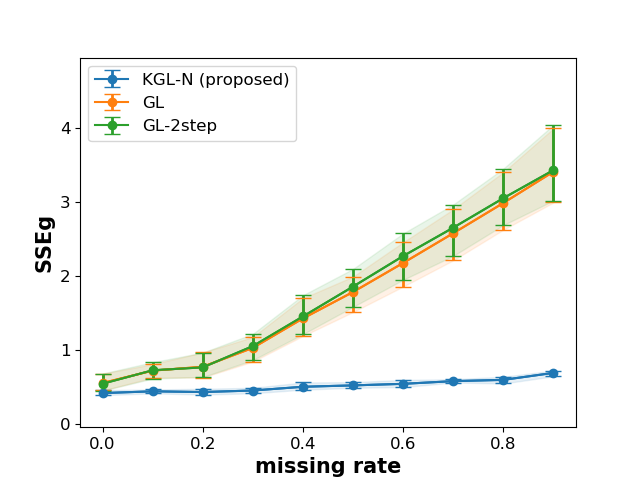}
     \subcaption{$\mathcal{G}_{\text{SBM}}$, independent data}
    \label{fig:exp_missing_SBM_SSEg}
\end{subfigure}
\begin{subfigure}[t]{0.25\textwidth}
    \centering
    \includegraphics[width = 1\linewidth]{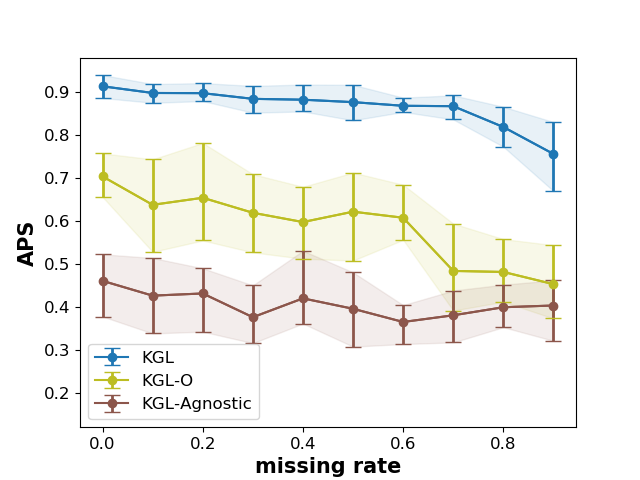}
    \subcaption{$\mathcal{G}_{\text{SBM}}$, dependent data}
    \label{fig:exp_missing_SBM_APS_2side}
\end{subfigure}%
\begin{subfigure}[t]{0.25\textwidth}
    \centering
    \includegraphics[width = 1\linewidth]{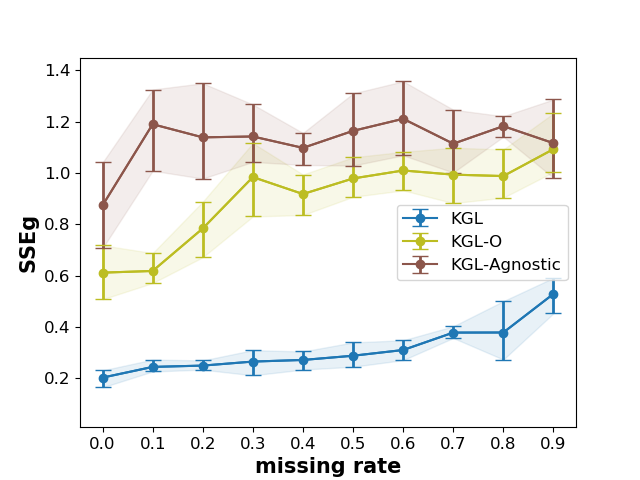}
     \subcaption{$\mathcal{G}_{\text{SBM}}$, dependent data}
    \label{fig:exp_missing_SBM_SSEg_2side}
\end{subfigure}%
\caption{The performance of recovering groundtruth graphs $\mathcal{G}_{\text{SBM}}$ in terms of APS and SSE$_\mathcal{G}$ from independent data (1st row) and dependent data (2nd row) with different rates of missing values in $\mathbf{Y}$.}  
\label{fig: syn_exp_missing}
\end{figure}

To examine the performance of learning a graph from incomplete data with $\mathbf{KGL}$ described in Section \ref{sec: missing}, we generate the mask matrix $\mathbf{M}$ indicating missing entries, where $\mathbf{M}_{ij} \overset{i.i.d.}{\sim} Bernoulli(1-r)$ and $r$ is the missing rate, i.e. $\mathbf{M}_{ij}$ has a probability of $r$ to be missing and has a value of 0. The preprocessed data with 0 replacing missing entries is $\mathbf{Y}_{m} = \mathbf{Y} \circ \mathbf{M}$, which is a natural choice in practice for the model candidates that cannot directly deal with missing entries, as the mean value of the entries in $\mathbf{Y}$ is 0 by design. We use $\mathbf{Y}_{m}$ as the input in the baseline models \textbf{GL} and \textbf{GL-2Step}. For \textbf{KGL-Agnostic} in Eq.(\ref{eq:baseline_KGL_ObsDep}), we also add a mask matrix $\mathbf{M}$ in the least-squares loss term to have a fair comparison. \par

We vary $r$ from 0 to 0.9, against which we plot the evaluation metrics, APS and SSE$_\mathcal{G}$, in Figure \ref{fig: syn_exp_missing}. The plots for $\mathcal{G}_{\text{ER}}$ and $\mathcal{G}_{\text{BA}}$ are in Appendix \ref{sec: appendix_figures_missing}. Similar to the noisy scenario, we repeat the experiments 10 times under the same settings of missing rate, random graph and model candidate and report the mean (the solid curves) as well as the 5th and 95th percentile (the error bars) of the evaluation metrics. \par

For $\mathcal{G}_{\text{SBM}}$, the proposed methods \textbf{KGL} and \textbf{KGL-N} can recover the groundtruth graphs reasonably well even when there are 80\% missing entries. For the independent data scenario, the performance of the baseline models with the preprocessed data $\mathbf{Y}_m$ drops steeply as the missing rate increases. Although it can still recover the groundtruth graphs with a high APS and low SSE$_\mathcal{G}$ when the missing rate is less than 20\%, the performance is not as good as \textbf{KGL-N}. By contrast, for \textbf{KGL-N}, the APS only drops by 0.1 from no missing entries to around 90\% missing entries, while SSE$_\mathcal{G}$ only increases by around 0.05. \par

For the dependent data scenario, all three model candidates can deal with missing data directly by adding a mask matrix in the least-squares loss term. Consequently, their curves are relatively stable as the missing rate increases from 0 to 80\%. However, the accuracy levels at which each of the model stabilises are different. The proposed method $\textbf{KGL}$, aware of both node-side and observation-side information, achieves the highest APS and lowest SSE$_{\mathcal{G}}$. By contrast, $\textbf{KGL-O}$, without access to the node-side information, can still recover a meaningful graph but with less correct edges and less accurate edge weights when the missing rate is less than 0.6. The performance of $\textbf{KGL-Agnostic}$ is not as good as $\textbf{KGL-O}$ because the smoothness term in Eq.(\ref{eq:baseline_KGL_ObsDep}) is not able to disentangle the influence of the observation-side dependency from the graph structure. \par

The performance of \textbf{KGL} does not differ much for different random graph models. Still, there is an improvement in recovering $\mathcal{G}_{\text{BA}}$ compared to the baseline graph learning models, as can be seen in Figure \ref{fig:exp_missing_BA_APS_2side} and Figure \ref{fig:exp_missing_BA_SSEg_2side} in Appendix \ref{sec: appendix_figures_missing}. \par

\subsection{Graph-Structured Matrix Completion}


\begin{figure}[t]
\centering
\begin{subfigure}[t]{0.25\textwidth}
    \centering
    \includegraphics[width = 1\linewidth]{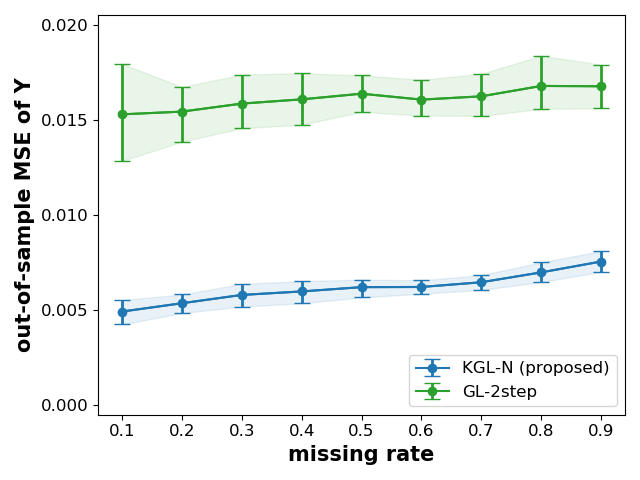}
    \subcaption{independent data}
    \label{fig:exp_missing_Y_indep}
\end{subfigure}%
\begin{subfigure}[t]{0.25\textwidth}
    \centering
    \includegraphics[width = 1\linewidth]{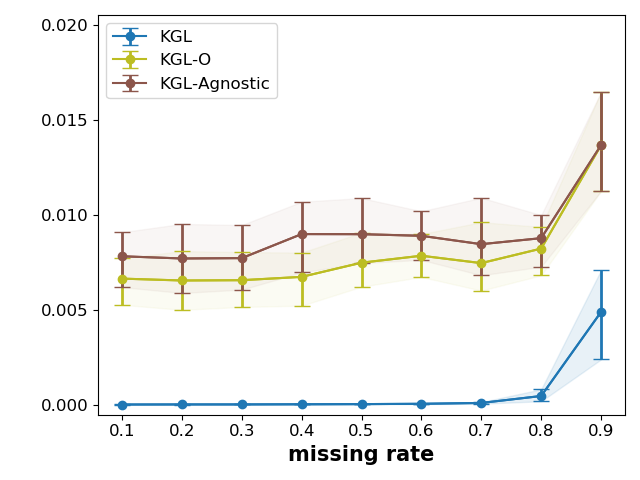}
     \subcaption{dependent data}
    \label{fig:exp_missing_Y_dep}
\end{subfigure}
\caption{The MSE of recovering missing entries in $\mathbf{Y}_m$. The results show the mean of 30 random experiments of 3 different types of graphs, i.e. 10 for each graph type.}
\label{fig:exp_missing_Y}
\end{figure}

In the experiment of learning a graph with incomplete data matrix in Section \ref{sec: syn_exp_missing_graph}, we are also interested in the performance of matrix completion. In Figure \ref{fig:exp_missing_Y}, we plot the MSE$_y$ of the missing entries (i.e. the out-of-sample MSE$_y$) that is averaged over all types of graphs against the varying missing rate. The proposed methods lead to much smaller errors compared to baseline models, for both independent and dependent data. 
It should be noted that $\textbf{GL}$ does not offer a mechanism for inferring missing data, hence is not included in this experiment.

\subsection{Learning a Graph of Different Sizes}

The graph learning performance of the proposed method varies with the size of graphs and number of observations. As shown in Figure \ref{fig:varying_size}, a hundred observations are sufficient to recover a small graph with $m = 20$ nodes with a high accuracy. When the graph size increases, the number of the observations required for $\textbf{KGL}$ to achieve a high APS increases roughly exponentially. On the other hand, when the number of observations increases, the variance of APS of the learned graphs decreases.  

\begin{figure}[t]
    \centering
     \includegraphics[width = 0.55\linewidth]{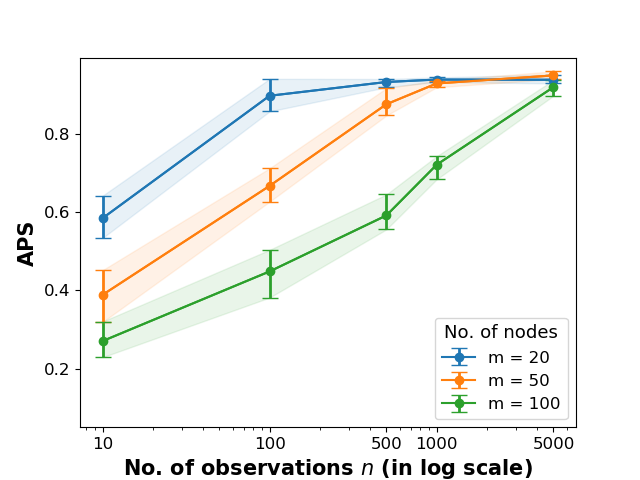}
    \caption{The performance of learning graphs of different sizes $m$ against varying number of observations $n$ with the proposed model \textbf{KGL}.  The results show the mean of 30 random experiments of 3 different types of graphs, i.e. 10 for each graph type.} 
    \label{fig:varying_size}
\end{figure}

\subsection{Impact of Regularisation Hyperparameters}
\label{sec:syn_exp_hyperparameters}
Three hyperparameters are involved in $\mathbf{KGL}$ and its variants. As introduced in Section \ref{sec: methods_KGL}, $\lambda > 0$ controls the complexity of the functional learning and prevents overfitting; $\rho > 0$ controls the relative importance of graph learning compared to functional learning, and at the same time determines the smoothness of the predicted data $\mathbf{\hat{y}}$ over $\mathbf{L} \otimes \mathbf{K}_z^{\dagger}$; Finally, $\psi > 0$ controls the Frobenius ($\ell_2$) norm of the graph Laplacian which, together with the trace ($\ell_1$) constraint, bears similarity to an elastic net regularisation \cite{zou2005regularization}. The larger the $\psi$, the less sparse the graph with more uniform edge weights.
The accuracy in learning the graph Laplacian $\mathbf{L}$ and inferring the data matrix $\mathbf{Y}$ is determined by the combination of these three hyperparameters, which is not straightforward to visualise and analyse at the same time. 
Fortunately, we may still gain some insights by examining their distinct effects separately. \par

Firstly, $\psi$ should be chosen according to the prior belief of the graph sparsity defined as the number of edges with non-zero weights. As shown in Figure \ref{fig:exp_ergu_G_heatmap} (in Appendix \ref{sec:appendix_regularisation}), when $\psi \rightarrow 0$, the learned graph contains only a few most significant edge. Due to the constraint on the sum of edge weights, i.e. $\text{tr}(\mathbf{L}) = m$, the total weights $m$ are allocated to a few significant edges when $\psi$ is small. When $\psi \rightarrow \infty$, the learned graph becomes fully connected with equal edge weights. In the synthetic experiment where we have knowledge of the groundtruth graph, the best accuracy is obtained when the sparsity coincides with the groundtruth graph. \par

The value of $\psi$, on the other hand, has little effect on the accuracy of predicting the missing entries in $\mathbf{Y}$, as the update step of $\mathbf{A}$ does not involve the term with $\psi$. As shown in Figure \ref{fig:exp_regularisation_path_accuracy}(a)-(d), the out-of-sample MSE$_{y}$ is determined by the combination of $\lambda$ and $\rho$. We first notice that the error is the same when $\lambda > 10^{2}$, showing that we overly penalise the function complexity in this case. Indeed, when $\lambda \rightarrow \infty$, the function is overly smooth such that the entries of the coefficient matrix $\mathbf{A}$ are all zero leading to the entries of prediction $\mathbf{\hat{Y}}$ being all zero as well. This also happens when $\rho \rightarrow \infty$, where the vector form of the prediction $\mathbf{\hat{y}}$ is forced to be overly smooth on $\mathbf{L} \otimes \mathbf{K}_z^{\dagger}$. In particular, when $\mathbf{K}_z^{\dagger} = \mathbf{I}_n$, every row vector of $\mathbf{\hat{Y}}$ (i.e. the predicted graph signal) has constant entries, as a result of minimising the term $\text{Tr}(\mathbf{\hat{Y}}^{\top}\mathbf{K}_z^{\dagger}\mathbf{\hat{Y}}\mathbf{L})$ to zero. On the other hand, when $\mathbf{K}_z^{\dagger} \neq \mathbf{I}_n$, all the entries in $\mathbf{\hat{Y}}$ has constant values such that this term is minimised to zero. \par

The ranges of values of $\alpha$ and $\rho$ for which the out-of-sample MSE$_{y}$ is the smallest generally coincide with that for which the APS is the largest in Figure \ref{fig:exp_regularisation_path_accuracy} (in Appendix \ref{sec:appendix_regularisation}), e.g. when $\psi = 10^{-5}$, $\alpha = 10^{-2}$ and $\rho = 10^{-2}$. However, the out-of-sample MSE$_{y}$ is generally small when $\alpha < 0.1$ and $\rho < 0.1$. This is understandable as the function could be very complex with little penalisation, but this does not guarantee good performance in recovering the groundtruth graph. 

\section{Real-World Experiments}

\subsection{Swiss Temperature Data}

In this experiment, we test the proposed models in learning a meteorological graph of 89 weather stations in  Switzerland from the incomplete temperature data\footnote{The data are obtained from \url{https://www.meteoswiss.admin.ch/home/climate/swiss-climate-in-detail/climate-normals/normal-values-per-measured-parameter.html}.}. The raw data matrix contains 12 rows representing the temperatures of 12 months that are averaged over 30 years from 1981 to 2010 and 89 columns representing 89 measuring stations. The raw data are preprosessed such that each row has a zero mean. \par

To have a fair comparison, we deliberately omit a portion of the data as input for the model candidates and treat as groundtruth the learned graph obtained from $\textbf{GL}$ using the complete 12-month data.
Specifically, we only use the first three-month temperature records (i.e. the first three rows) to learn a graph. To test the performance in missing scenario, we further generate 
the mask matrix $M$ with different rates of missing values, as described in Section \ref{sec: syn_exp_missing_graph}, and apply them to the three-month data. Similar to the synthetic experiment, we choose the hyperparameters in models that yield a graph density of 30\% (i.e. keeping 30\% most significant edges) to make the learned graphs comparable.\par

The altitude of a weather station is a useful node-side information for predicting temperature and learning a meteorological graph. Therefore, the corresponding kernel matrix $\mathbf{K}_x$ is obtained from an RBF kernel evaluated at the altitudes of each pair of weather stations for use in \textbf{KGL} and \textbf{KGL-N}. Monthly time-stamps correspond to the known observation-side information. The bandwidth parameter is chosen according to the median heuristic \cite{garreau2017large}. As described in Section \ref{sec: syn_exp_general_settings}, $\mathbf{K}_z$ is thus obtained from an RBF kernel evaluated at three time-stamps for three-month graph signals and used as input in \textbf{KGL}, \textbf{KGL-Agnostic} and \textbf{KGL-O}. For \textbf{GL} and \textbf{GL-2step}, no side information is used. The hyperparameters are tuned via a grid search and highest performance achieved by the best set of hyperparameters is reported. For the real-world scenario where a groundtruth graph is not easy to obtain, the hyperparameter can be chosen according to the results in Section \ref{sec:syn_exp_hyperparameters}.  \par

\begin{figure}[t]
\centering
\begin{subfigure}[t]{0.25\textwidth}
    \centering
    \includegraphics[width = 1\linewidth]{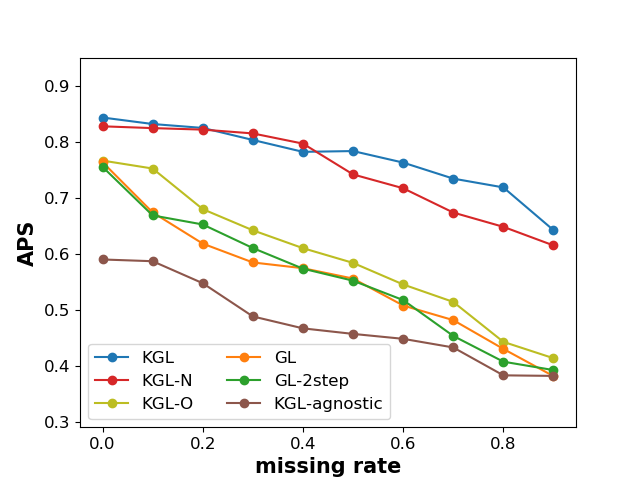}
    \subcaption{APS}
    \label{fig:real_temp_aps}
\end{subfigure}%
\begin{subfigure}[t]{0.25\textwidth}
    \centering
    \includegraphics[width = 1\linewidth]{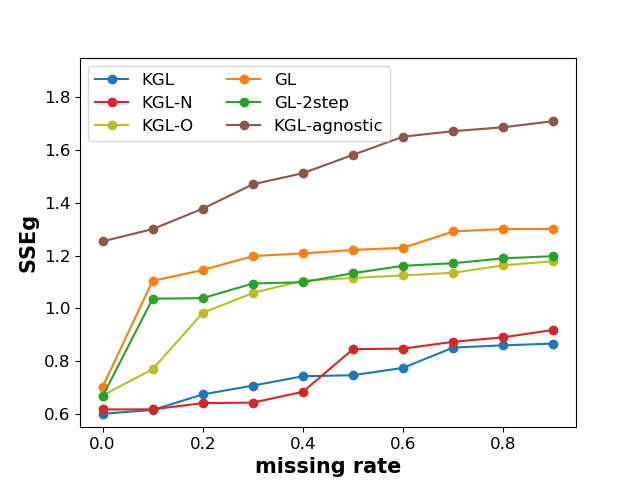}
     \subcaption{SSE$_\mathcal{G}$}
    \label{fig:real_temp_nmse}
\end{subfigure}
\caption{The performance of learning a meteorological graph of 89 Swiss weather stations from the incomplete temperature data with various missing rates.}
\label{fig:real_temp}
\end{figure}

We present the results in Figure \ref{fig:real_temp}. In terms of both APS and SSE$_\mathcal{G}$,  \textbf{KGL} and  \textbf{KGL-N} outperform the other candidates. This indicates that altitude is a reliable node-side covariate with which we can learn a meaningful meteorological graph despite a small number of signals. When there are more missing values, \textbf{KGL}, with the known temporal information, slightly outperforms \textbf{KGL-N}. However, since we only have three-month signals, the temporal information is less predictive. Since the groundtruth graph is learned from $\textbf{GL}$ with 12-month signals, the recovering ability of $\textbf{GL}$ and $\textbf{GL-2step}$ is not far behind when there is no missing values in three-month data, but drops sharply with an increasing missing rate. Compared to \textbf{KGL-O}, the poor performance of $\textbf{KGL-Agnostic}$ indicates that the imprecise smoothness term in Eq.(\ref{eq:baseline_KGL_ObsDep}) is the main reason that we cannot recover an annual meteorological graph with only three-month temperature records, as both of them are agnostic to the node-side information.

\subsection{Sushi Review Data}

In this experiment we will evaluate the performance of our proposed methods by comparing the recovered graphs with groundtruth using the Sushi review data collected in \cite{kamishima2003nantonac}. The authors tasked 5000 reviewers to rate 10 out of 100 sushis randomly with a score from 1 (least preferred) to 5 (most preferred);  reviews for each sushi are treated as one graph signal in this experiment. For each reviewer, we have 10 descriptive features which cover demographical information about the reviewers, such as age, gender and the region the reviewer currently lives in. We also have 7 attributes describing each sushi, including its oiliness in taste, normalised price and its grouping (for example, red-meat fish sushi, white-meat fish sushi or shrimp sushi). We will treat the grouping attribute as the underlying groundtruth label for each sushi and not use it in the KGL algorithm. 

We will consider 32 sushis from 5 sushi groups, namely red-meat (7 sushis), clam (6 sushis), blue-skinned fish (8 sushis), vegetable (6 sushis) and roe sushi (5 sushis). These are treated as the groundtruth labels. Our goal will be to recover a graph of sushis which contains clusters corresponding to these group labels (while omitting the group attribute from the node-side information, i.e. we retain only 6 remaining attributes). 

We pick an increasing number of reviewers at random for our experiment. This is to demonstrate how the algorithm performs under different number of signals. After preprocessing, we arrive to a data matrix with 32 columns, each representing a type of sushi, and rows representing each reviewer's rating to the sushis. This is a sparse matrix with an average sparsity of $74 \%$. We run \textbf{KGL}, \textbf{KGL-N}, \textbf{KGL-O}, \textbf{GL}, \textbf{GL-2step} and \textbf{KGL Agnostic} to obtain a graph of sushis. To evaluate the result quantitatively, we compute the \textit{normalised mutual information} (NMI) between the cluster assignments obtained by applying spectral clustering \cite{von2007tutorial} to the recovered graphs and the underlying sushi grouping. NMI is used to measure the agreement between two grouping assignments, 0 indicating no mutual information while 1 indicating perfect correlation. To emphasise that node-side attributes alone are not sufficient to recover the groundtruth label, we also applied clustering on the RBF graph obtained from remaining six sushi attributes, which resulted in an NMI score of 0.34.

\begin{figure}[t]
    \centering
    \includegraphics[width=0.55\linewidth]{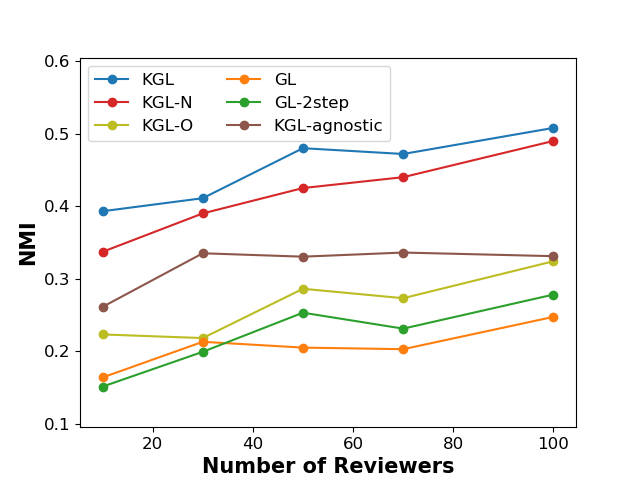}
    \caption{Agreement between cluster assignments on learned sushi graphs and the withheld sushi group attributes. Graphs are learned from incomplete sushi review data.}
    \label{fig:sushis}
\end{figure}

Figure \ref{fig:sushis} illustrated our results. \textbf{KGL} was the best performer, followed by \textbf{KGL-N}, and both significantly outperformed \textbf{KGL-Agnostic}, \textbf{KGL-O}, \textbf{GL-2step} and \textbf{GL}. Moreover, both \textbf{KGL} and \textbf{KGL-N} outperformed the case where we solely use the RBF Graph. The results demonstrate the merit of our proposed methods in incorporating side information for graph recovery, in particular the observation-side information (i.e. the reviewers' information) to capture the dependency between the observed signals.

\section{Conclusion}

In this paper, we have revisited the smooth graph signals from a functional viewpoint and proposed a kernel-based graph learning framework that can integrate node-side and observation-side covariates. Specifically, we have designed a novel notion of smoothness of graph signals over the Kronecker product of two graph Laplacian matrices and combined it with a Kronecker product kernel regression of graph signals in order to capture the two-side dependency. We have shown the effectiveness and efficiency of the proposed method, via extensive synthetic and real-world experiments, demonstrating its usefulness in learning a meaningful topology from noisy, incomplete and dependent graph signals. Although we have proposed a fast implementation exploiting the Kronecker structure of kernel matrices, the computational complexity remains cubic in the maximum of the number of nodes and the number of signals. Hence, a natural future direction is to further reduce the computational complexity with the state-of-the-art methods for large-scale kernel-based learning, such as random Fourier features. Another interesting direction is to develop a generative graph learning model based upon the framework presented here, using connections between Gaussian processes and kernel methods.


%



\clearpage
\appendices
\section{Kronecker product kernel regression}
\label{sec: appendix_generative_model_interpretation}
Taking a Bayesian viewpoint, we
provide an interpretation of 
Eq.(\ref{eq: kronecker_regression}). From Eq.(\ref{eq: y_generative_all}), maximising the log-posterior of the coefficient vector $\mathbf{a}$ leads to
the objective in Eq.(\ref{eq: kronecker_regression}):
\begin{equation}
\nonumber
\begin{split}
 \max_{\mathbf{a}} & \log p(\mathbf{a}|\mathbf{y}) \\ 
 = \max_{\mathbf{a}} & \log p(\mathbf{y}|\mathbf{a}) + \log p(\mathbf{a}) \\
=  \max_{\mathbf{a}} &  - \big( \mathbf{y} -  (\mathbf{K}_x \otimes \mathbf{K}_z) \mathbf{a} \big)^{\top} \big( \mathbf{y} -  (\mathbf{K}_z \otimes \mathbf{K}_z \mathbf{a} \big) \\
& - \lambda \mathbf{a}^{\top} (\mathbf{K}_x \otimes \mathbf{K}_z)\mathbf{a} \\
= \min_{\mathbf{a}} & ||\mathbf{y} -  (\mathbf{K}_x \otimes \mathbf{K}_z)\mathbf{a} ||^2_2 + \lambda \mathbf{a}^{\top} (\mathbf{K}_x \otimes \mathbf{K}_z)\mathbf{a} \\
= \min_{\mathbf{A}} & ||\mathbf{Y} - \mathbf{K}_z \mathbf{A}  \mathbf{K}_x||_F^2  + \lambda \text{Tr}(\mathbf{K}_z \mathbf{A}  \mathbf{K}_x \mathbf{A}^{\top})
\end{split}
\end{equation}%
where $\lambda$ is some constant parameter proportional to the variance of the noise $\sigma_\epsilon^2$ in Eq.(\ref{eq: y_generative_data_likelhood}).
 \par

\section{Kronecker Product Laplacian-like Operator}
\label{sec: appendix_Kron_Laplacian}
We define a Laplacian-like operator with a Kronecker product structure $\mathbf{L}_{\otimes} =  \mathbf{L}_x \otimes \mathbf{L}_z$ for the data matrix with both node-side and observation-side dependency. Although $\mathbf{L}_{\otimes}$ may have positive off-diagonal entries and thus may not be a valid graph Laplacian matrix, it satisfies the following properties and the notion of frequencies of $\mathbf{y}$ can be defined upon $\mathbf{L}_{\otimes}$:
\begin{itemize}
    \item $\mathbf{L}_{\otimes}$ is symmetric and $\mathbf{L}_{\otimes}\cdot\mathbf{1} = \mathbf{0}$;
    \item $\mathbf{L}_{\otimes}$ admits the eigendecomposition 
    \begin{equation}
    \begin{split}
    \nonumber
        \mathbf{L}_{\otimes} & = \mathbf{L}_x \otimes \mathbf{L}_z \\
        & = (\mathbf{U}_x \otimes \mathbf{U}_z) (\boldsymbol{\Lambda}_x \otimes \boldsymbol{\Lambda}_z) (\mathbf{U}_x \otimes \mathbf{U}_z)^{\top} 
    \end{split}
    \end{equation}%
where $\mathbf{U}_x$ and $\boldsymbol{\Lambda}_x$, and $\mathbf{U}_z$ and $\boldsymbol{\Lambda}_z$, are the eigenvector and eigenvalue matrices of the Laplacian matrices $\mathbf{L}_x$ and $\mathbf{L}_z$, respectively, and
$\mathbf{U} = (\mathbf{U}_x \otimes \mathbf{U}_z)$ is an orthogonal matrix and $\boldsymbol{\Lambda} = (\boldsymbol{\Lambda}_x \otimes \boldsymbol{\Lambda}_z)$ is a diagonal matrix with real entries. 
\end{itemize}%
We can also obtain a two-dimensional graph Fourier transform $\mathbf{\check{Y}}$ of $\mathbf{Y}$ as in \cite{monti2017geometric}:
\begin{equation}
\nonumber
\begin{split}
    \text{vec}(\mathbf{\check{Y}}) = (\mathbf{U}_x \otimes \mathbf{U}_z)^{\top} \text{vec}(\mathbf{Y}). 
\end{split}
\end{equation}

\section{Proof of positive semi-definiteness}
\label{sec: appendix_psd}
\begin{lemma}
\label{lemma: optiA_matrix_psd}
The matrix $\mathbf{C} = \mathbf{K}^2 + \lambda \mathbf{K} + \rho \mathbf{S} \otimes \mathbf{K}_z$ is positive semi-definite.
\end{lemma}

\begin{proof}
To prove $\mathbf{C}$ is positive semi-definite (p.s.d.), it suffices to show that the matrices $\mathbf{K}^2$, $\mathbf{K}$ and $\mathbf{S} \otimes \mathbf{K}_z$ are p.s.d. \par
As $\mathbf{K}_x$ and $\mathbf{K}_z$ are kernel matrices constructed by pairwise evaluations from two reproducing kernels $\kappa_\mathcal{X}$ and  $\kappa_\mathcal{Z}$, they are p.s.d. The Kronecker product $\mathbf{K}$ is p.s.d, which is easy to prove from the eigendecomposition:
\begin{equation}
\begin{split}
\label{eq: append_psd_eigen}
    \mathbf{K} & = \mathbf{K}_x \otimes \mathbf{K}_z \\
    & = (\mathbf{U}_x \boldsymbol{\Lambda}_x \mathbf{U}_x^{\top} ) \otimes(\mathbf{U}_z \boldsymbol{\Lambda}_z \mathbf{U}_z^{\top} ) \\
    & = (\mathbf{U}_x \otimes \mathbf{U}_z) (\boldsymbol{\Lambda}_x \otimes \boldsymbol{\Lambda}_z) (\mathbf{U}_x \otimes \mathbf{U}_z)^{\top} 
\end{split}
\end{equation}%
where $\mathbf{U} = (\mathbf{U}_x \otimes \mathbf{U}_z)$ is an orthogonal matrix and $\boldsymbol{\Lambda} = (\boldsymbol{\Lambda}_x \otimes \boldsymbol{\Lambda}_z)$ is a diagonal matrix with non-negative real entries. \par
Next, $\mathbf{K}^2 = \mathbf{K}\mathbf{K}$ is also p.s.d., as it is a product of commuting matrices and $\mathbf{K}^2$ preserves symmetry \cite{meenakshi1999product}. \par
Finally, denote the column vectors of $\mathbf{K}_x$ as $[\mathbf{k}_1, \mathbf{k}_2, \dots, \mathbf{k}_m]$, the weighted adjacency matrix as $\mathbf{W}$, and the non-negative edge weight between node $j$ and $j'$ as $w_{jj'}$. We have
\begin{equation}
\nonumber
    \mathbf{S} = \mathbf{K}_x \mathbf{L}\mathbf{K}_x =  \sum_{j \neq j'} w_{jj'} (\mathbf{k}_j - \mathbf{k}_{j'})(\mathbf{k}_j - \mathbf{k}_{j'})^{\top}
\end{equation}%
where $w_{jj'} \geq 0$. $\mathbf{S}$ can then be viewed as a weighted covariance matrix, which is symmetric and p.s.d. Therefore, $\mathbf{S} \otimes \mathbf{K}_z$ is p.s.d. following the same argument as for $\mathbf{K}$.

\end{proof}


\section{Derivation of Gradient in Eq.(\ref{eq: SKGL_gradient_noK})}
\label{sec: appendix_missing_optimisation}


In this section, we show the derivation of Eq.(\ref{eq: SKGL_gradient_noK}), i.e. the gradient for updating $\mathbf{A}$ in the missing values scenario. Recall that the objective function is 
\begin{equation}
\nonumber
\begin{split}
    J_{\mathbf{L}}(\mathbf{A})  = & ||\mathbf{M} \circ(\mathbf{Y} - \mathbf{K}_z \mathbf{A}  \mathbf{K}_x)||_F^2  + \lambda \text{Tr}(\mathbf{K}_z \mathbf{A}  \mathbf{K}_x \mathbf{A}^{\top}) \\ & + \rho \text{Tr}(\mathbf{A} \mathbf{K}_x \mathbf{L} \mathbf{K}_x \mathbf{A}^{\top} \mathbf{K}_z).
\end{split}
\end{equation}%
With the following standard linear algebra identities for any matrices $\mathbf{D}, \mathbf{E}, \mathbf{F}$ 
\begin{itemize}
    \item $||\mathbf{D} \circ \mathbf{E}||^2_F = \text{Tr}\big( (\mathbf{D} \circ \mathbf{E})^{\top}  (\mathbf{D} \circ \mathbf{E}) \big) = \text{Tr}\big( \mathbf{E}^{\top}( \mathbf{D} \circ \mathbf{E}) \big)$
    \item $\text{vec}(\mathbf{D} \circ \mathbf{E}) = \text{vec}(\mathbf{D}) \circ \text{vec}(\mathbf{E})$
    \item $\text{Tr}(\mathbf{D}^{\top} \mathbf{E}) = \text{vec}(\mathbf{D})^{\top} \text{vec}(\mathbf{E})$
    \item $\text{vec}(\mathbf{DEF}) = (\mathbf{F}^{\top} \otimes \mathbf{D})\text{vec}(\mathbf{E})$
\end{itemize}%
the first term of $J_{\mathbf{L}}(\mathbf{A})$ becomes
\begin{subequations}
\nonumber
\begin{align}
    &  ||\mathbf{M} \circ (\mathbf{Y} - \mathbf{K}_z \mathbf{A}  \mathbf{K}_x)||_F^2  \nonumber \\
    = &  \text{Tr} \Big( (\mathbf{Y} - \mathbf{K}_z \mathbf{A} \mathbf{K}_x)^{\top} \big( \mathbf{M} \circ (\mathbf{Y} - \mathbf{K}_z \mathbf{A} \mathbf{K}_x) \big) \Big) \nonumber 
\end{align}
\end{subequations}%
By dropping constant terms, we have
\begin{subequations}
\begin{align}
     &  ||\mathbf{M} \circ (\mathbf{Y} - \mathbf{K}_z \mathbf{A}  \mathbf{K}_x)||_F^2  \nonumber \\
     = & \text{Tr} \Big( -2(\mathbf{K}_z \mathbf{A} \mathbf{K}_x)^{\top} (\mathbf{M} \circ \mathbf{Y}) \Big) \nonumber \\ & +  \text{Tr} \Big( (\mathbf{K}_z \mathbf{A} \mathbf{K}_x)^{\top} \big(\mathbf{M} \circ (\mathbf{K}_z \mathbf{A} \mathbf{K}_x) \big) \Big) \nonumber  \\
    =  & -2  \text{vec}(\mathbf{A})^{\top} (\mathbf{K}_x \otimes \mathbf{K}_z) \text{vec}(\mathbf{M} \circ \mathbf{Y}) \nonumber \\ & +  \text{vec} (\mathbf{K}_z \mathbf{A} \mathbf{K}_x)^{\top} \big( \text{vec}(\mathbf{M}) \circ \text{vec}(\mathbf{K}_z \mathbf{A} \mathbf{K}_x) \big) \nonumber  \\
    = & -2 \mathbf{a}^{\top} \mathbf{K} \text{vec}(\mathbf{M} \circ \mathbf{Y}) +  \mathbf{a}^{\top} \mathbf{K} \big( \mathbf{m} \circ (\mathbf{K} \mathbf{a} ) \big) \nonumber 
\end{align}
\end{subequations}%
where $\mathbf{a} = \text{vec}(\mathbf{A})$, $\mathbf{m} = \text{vec}(\mathbf{M})$ and $\mathbf{K} = \mathbf{K}_x \otimes \mathbf{K}_z$. Putting it back to the objective and recognising the fact that $\mathbf{d} \circ \mathbf{e} = \text{diag}(\mathbf{d})\mathbf{e}$ for two vectors $\mathbf{d}$ and $\mathbf{e}$, we have 
\begin{equation}
\nonumber
\begin{split}
   J_{\mathbf{L}}(\mathbf{a}) = & -2 \mathbf{a}^{\top} \mathbf{K} \text{vec}(\mathbf{M} \circ \mathbf{Y}) +  \mathbf{a}^{\top} \mathbf{K} \big( \mathbf{m} \circ (\mathbf{K} \mathbf{a} )\big) \\ 
   & \quad + \lambda \mathbf{a}^{\top} \mathbf{K} \mathbf{a} + \rho  \mathbf{a}^{\top} (\mathbf{S} \otimes \mathbf{K}_z) \mathbf{a} \\
   = & -2 \mathbf{a}^{\top} \mathbf{K} \text{vec}(\mathbf{M} \circ \mathbf{Y}) +  \mathbf{a}^{\top} \mathbf{K} \text{diag}(\mathbf{m}) \mathbf{K} \mathbf{a}  \\ 
   & \quad + \lambda \mathbf{a}^{\top} \mathbf{K} \mathbf{a} + \rho  \mathbf{a}^{\top} (\mathbf{S} \otimes \mathbf{K}_z) \mathbf{a}
\end{split}   
\end{equation}%
where $\mathbf{S} = \mathbf{K}_x \mathbf{L} \mathbf{K}_x$. We thus obtain the gradient for deriving Eq.(\ref{eq: SKGL_gradient}) such that
\begin{subequations}
\nonumber
\begin{align}
    \nabla J_{\mathbf{L}}(\mathbf{a}) = & -  \mathbf{K} \text{vec}(\mathbf{M} \circ \mathbf{Y}) +  \mathbf{K} \text{diag}(\mathbf{m}) \mathbf{K} \mathbf{a} \\
    & \quad + \lambda \mathbf{K}  \mathbf{a} + \rho \big( \mathbf{S} \otimes \mathbf{K}_z \big) \mathbf{a}.
\end{align}
\end{subequations}


\section{Additional Results for Synthetic Experiments}

\subsection{Learning ER and BA Graphs from Noisy Data}
\label{sec: appendix_figures_noisy}

Following the settings in Section \ref{sec:syn_exp_noisy}, we present in Figure \ref{fig:exp_noisy_ER_appendix} and Figure \ref{fig:exp_noisy_BA_appendix} the results of recovering $\mathcal{G}_{\text{ER}}$ and $\mathcal{G}_{\text{BA}}$ from independent data and dependent data with different noise levels, respectively.

\begin{figure}[]
\begin{subfigure}[t]{0.25\textwidth}
    \centering
    \includegraphics[width = 1\linewidth]{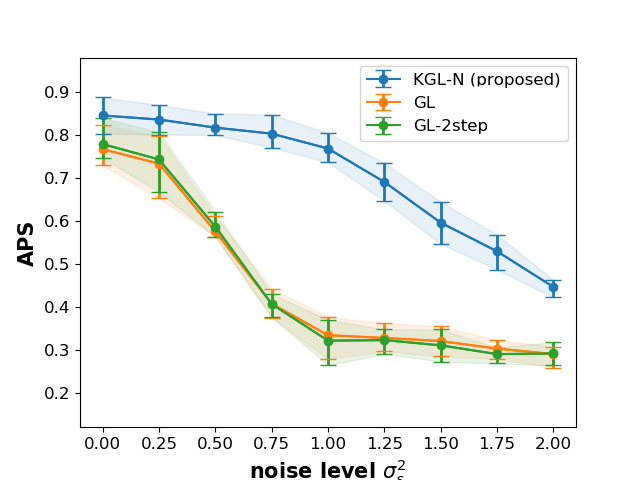}
    \subcaption{$\mathcal{G}_{\text{ER}}$, independent data}
    \label{fig:exp_noisy_ER_APS_2side}
\end{subfigure}%
\begin{subfigure}[t]{0.25\textwidth}
    \centering
    \includegraphics[width = 1\linewidth]{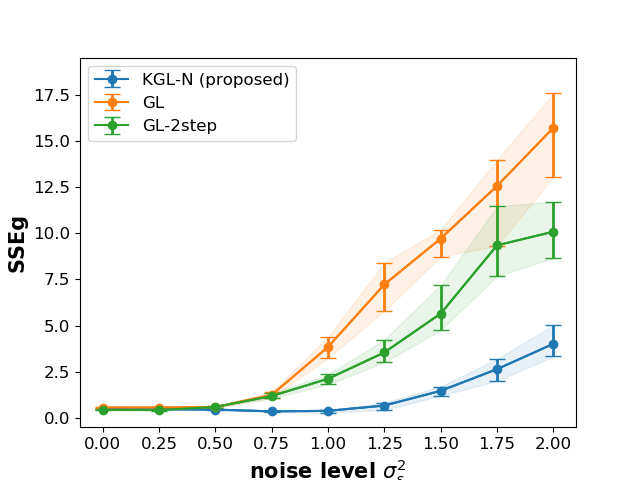}
     \subcaption{$\mathcal{G}_{\text{ER}}$, independent data}
    \label{fig:exp_noisy_ER_SSEg_2side}
\end{subfigure}
\begin{subfigure}[t]{0.25\textwidth}
    \centering
    \includegraphics[width = 1\linewidth]{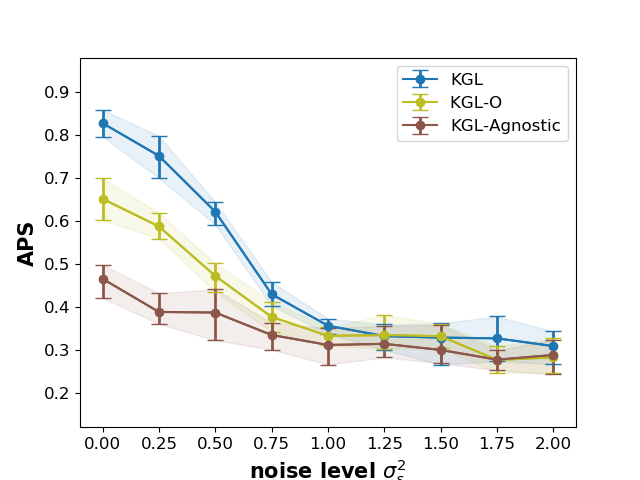}
    \subcaption{$\mathcal{G}_{\text{ER}}$, dependent data}
    \label{fig:exp_noisy_ER_APS_2side}
\end{subfigure}%
\begin{subfigure}[t]{0.25\textwidth}
    \centering
    \includegraphics[width = 1\linewidth]{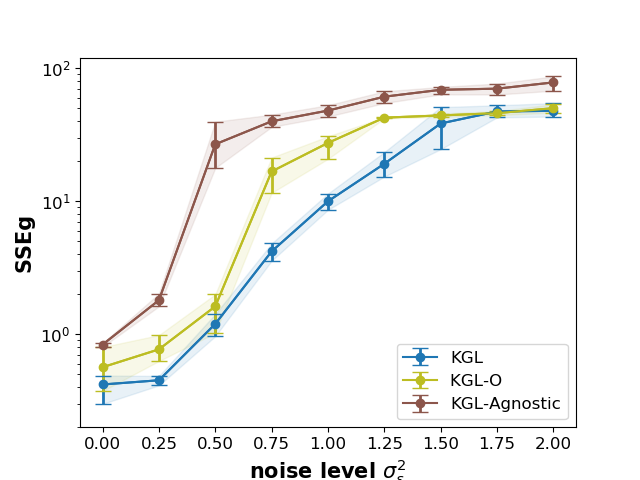}
     \subcaption{$\mathcal{G}_{\text{ER}}$, dependent data}
    \label{fig:exp_noisy_ER_SSEg_2side}
\end{subfigure}
\caption{The performance of recovering groundtruth graphs $\mathcal{G}_{\text{ER}}$ from independent data (1st row) and dependent data (2nd row) with different noise levels.}
\label{fig:exp_noisy_ER_appendix}
\end{figure}

\begin{figure}[]
\begin{subfigure}[t]{0.25\textwidth}
    \centering
    \includegraphics[width = 1\linewidth]{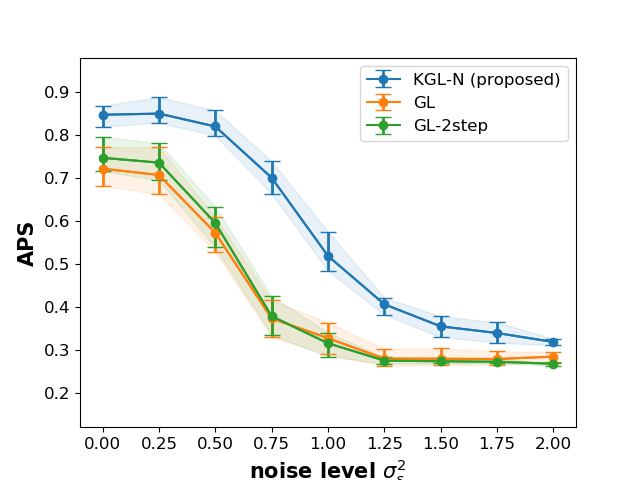}
    \subcaption{$\mathcal{G}_{\text{BA}}$, independent data}
    \label{fig:exp_noisy_BA_APS_2side}
\end{subfigure}%
\begin{subfigure}[t]{0.25\textwidth}
    \centering
    \includegraphics[width = 1\linewidth]{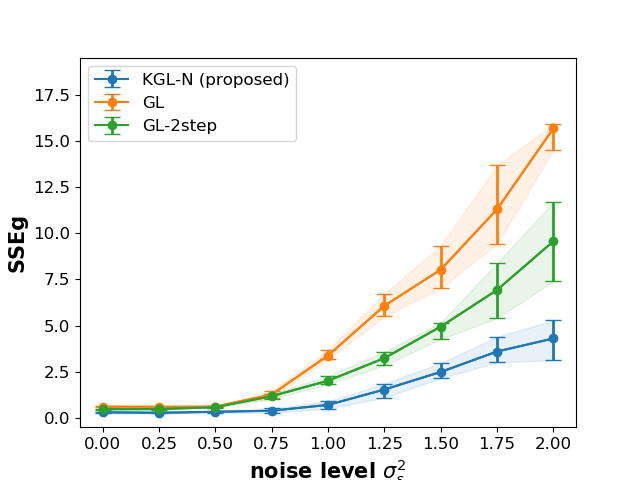}
     \subcaption{$\mathcal{G}_{\text{BA}}$, independent data}
    \label{fig:exp_noisy_BA_SSEg_2side}
\end{subfigure}
\begin{subfigure}[t]{0.25\textwidth}
    \centering
    \includegraphics[width = 1\linewidth]{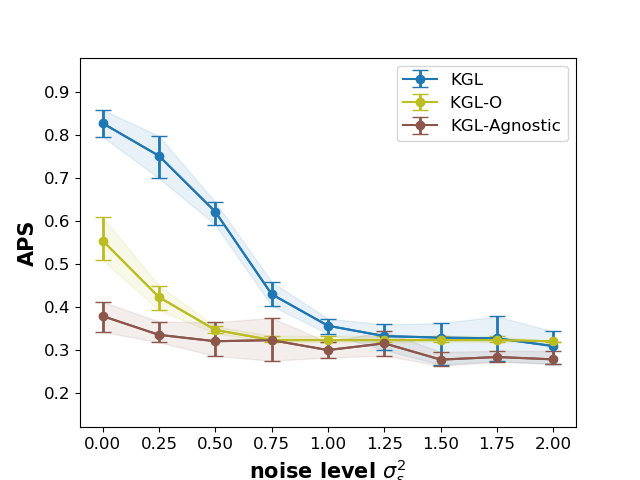}
    \subcaption{$\mathcal{G}_{\text{BA}}$, dependent data}
    \label{fig:exp_noisy_BA_APS_2side}
\end{subfigure}%
\begin{subfigure}[t]{0.25\textwidth}
    \centering
    \includegraphics[width = 1\linewidth]{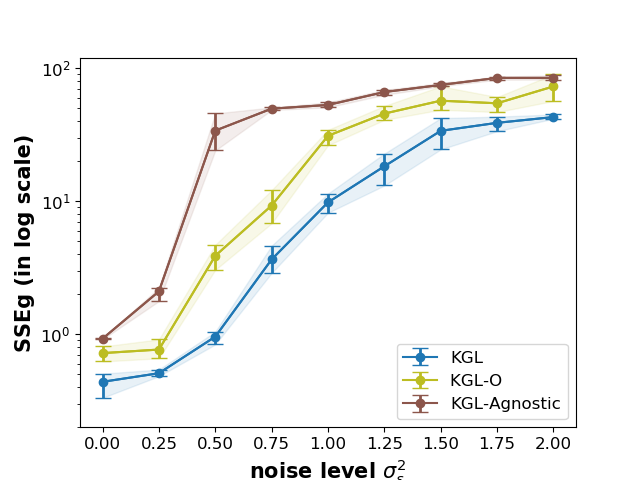}
     \subcaption{$\mathcal{G}_{\text{BA}}$, dependent data}
    \label{fig:exp_noisy_BA_SSEg_2side}
\end{subfigure}
\caption{The performance of recovering groundtruth graphs $\mathcal{G}_{\text{BA}}$ from independent data (1st row) and dependent data (2nd row) with different noise levels.}
\label{fig:exp_noisy_BA_appendix}
\end{figure}

\subsection{Learning ER and BA Graphs from Missing Data}
\label{sec: appendix_figures_missing}

Following the settings in Section \ref{sec:syn_exp_noisy}, we present in Figure \ref{fig:exp_missing_ER_appendix} and Figure  \ref{fig:exp_missing_BA_appendix} the results of recovering $\mathcal{G}_{\text{ER}}$ and $\mathcal{G}_{\text{BA}}$ from independent data and dependent data with different missing rates, respectively.

\begin{figure}[]
\centering
    \begin{subfigure}[t]{0.25\textwidth}
    \centering
    \includegraphics[width = 1\linewidth]{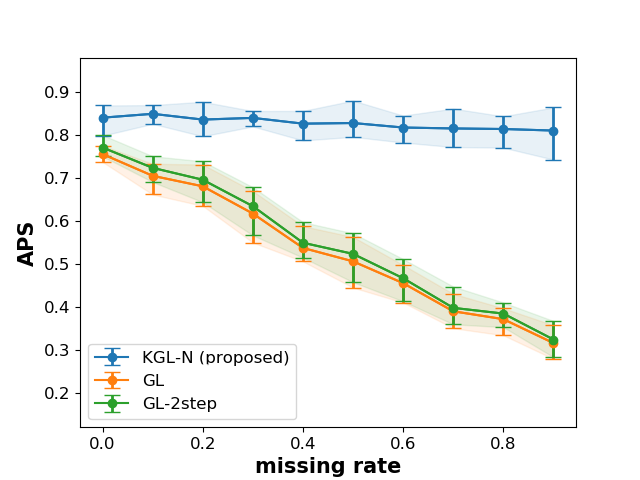}
    \subcaption{$\mathcal{G}_{\text{ER}}$, independent data}
    \label{fig:exp_missing_ER_APS}
\end{subfigure}%
\begin{subfigure}[t]{0.25\textwidth}
    \centering
    \includegraphics[width = 1\linewidth]{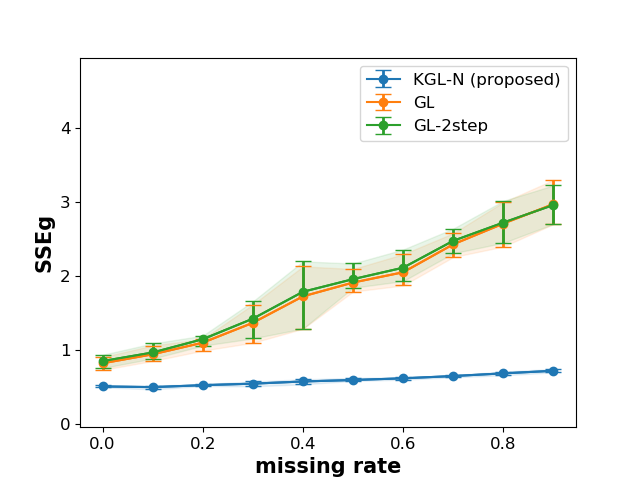}
     \subcaption{$\mathcal{G}_{\text{ER}}$, independent data}
    \label{fig:exp_missing_ER_SSEg}
\end{subfigure}
\begin{subfigure}[t]{0.25\textwidth}
    \centering
    \includegraphics[width = 1\linewidth]{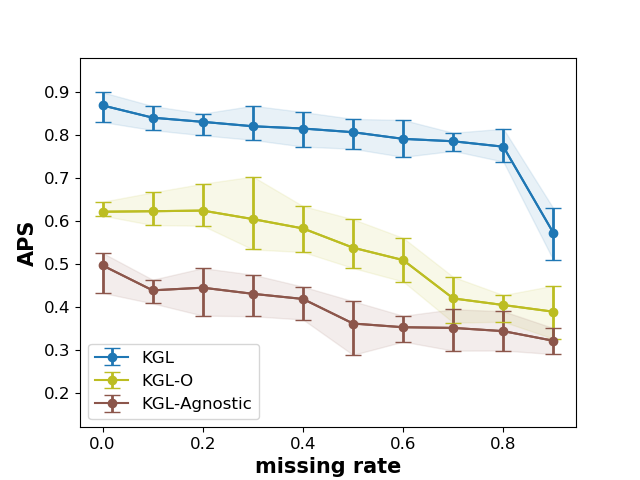}
    \subcaption{$\mathcal{G}_{\text{ER}}$, dependent data}
    \label{fig:exp_missing_ER_APS_2side}
\end{subfigure}%
\begin{subfigure}[t]{0.25\textwidth}
    \centering
    \includegraphics[width = 1\linewidth]{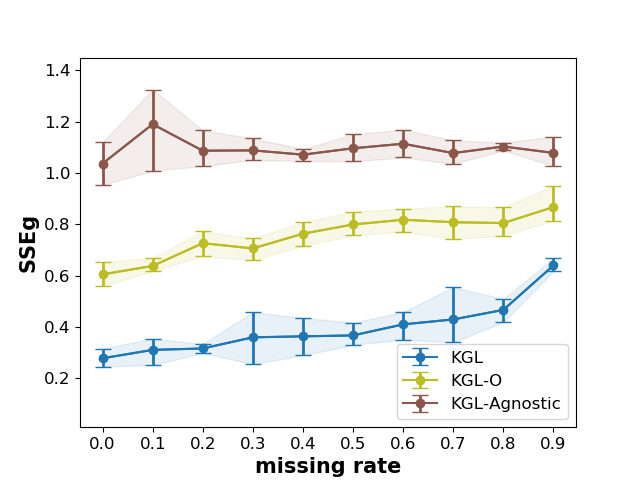}
     \subcaption{$\mathcal{G}_{\text{ER}}$, dependent data}
    \label{fig:exp_missing_ER_SSEg_2side}
\end{subfigure}
\caption{The performance of recovering groundtruth graphs $\mathcal{G}_{\text{ER}}$ from independent data (1st row) and dependent data (2nd row) with different rates of missing values in $\mathbf{Y}$.}
\label{fig:exp_missing_ER_appendix}
\end{figure}

\begin{figure}[]
\centering
\begin{subfigure}[t]{0.25\textwidth}
    \centering
    \includegraphics[width = 1\linewidth]{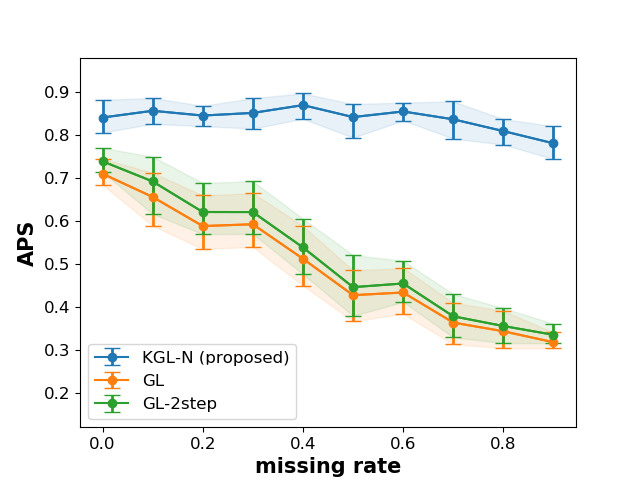}
    \subcaption{$\mathcal{G}_{\text{BA}}$, independent data}
    \label{fig:exp_missing_BA_APS}
\end{subfigure}%
\begin{subfigure}[t]{0.25\textwidth}
    \centering
    \includegraphics[width = 1\linewidth]{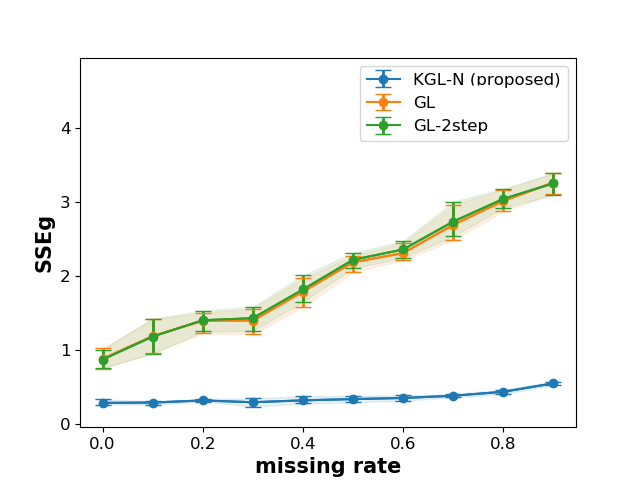}
     \subcaption{$\mathcal{G}_{\text{BA}}$, independent data}
    \label{fig:exp_missing_BA_SSEg}
\end{subfigure}
\begin{subfigure}[t]{0.25\textwidth}
    \centering
    \includegraphics[width = 1\linewidth]{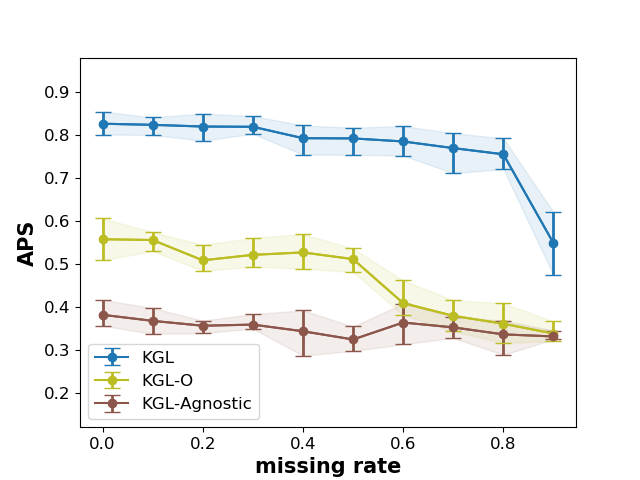}
    \subcaption{$\mathcal{G}_{\text{BA}}$, dependent data}
    \label{fig:exp_missing_BA_APS_2side}
\end{subfigure}%
\begin{subfigure}[t]{0.25\textwidth}
    \centering
    \includegraphics[width = 1\linewidth]{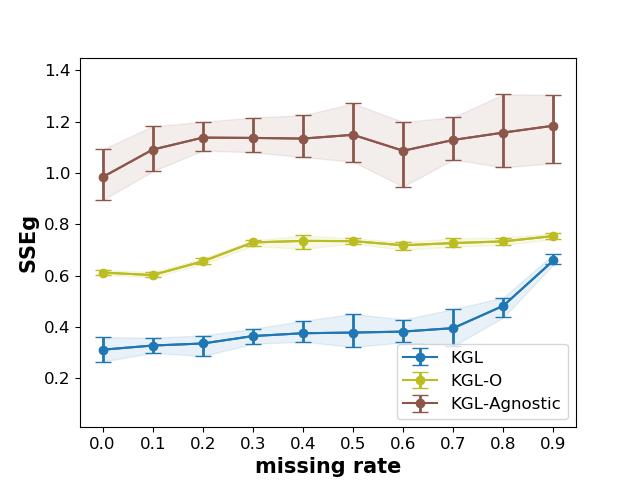}
     \subcaption{$\mathcal{G}_{\text{BA}}$, dependent data}
    \label{fig:exp_missing_BA_SSEg_2side}
\end{subfigure}
\caption{The performance of recovering groundtruth graphs $\mathcal{G}_{\text{BA}}$ from independent data (1st row) and dependent data (2nd row) with different rates of missing values in $\mathbf{Y}$.}
\label{fig:exp_missing_BA_appendix}
\end{figure}

\subsection{Impact of Regularisation Hyperparameters}
\label{sec:appendix_regularisation}

Figure \ref{fig:exp_ergu_G_heatmap} and Figure \ref{fig:exp_regularisation_path_accuracy} illustrate the learning performance with respect to the 
three hyperparameters in the proposed \textbf{KGL} model in Section \ref{sec:syn_exp_hyperparameters}.

\begin{figure*}[]
\centering
\begin{subfigure}[t]{0.19\textwidth}
    \centering
    \includegraphics[width = 1\linewidth]{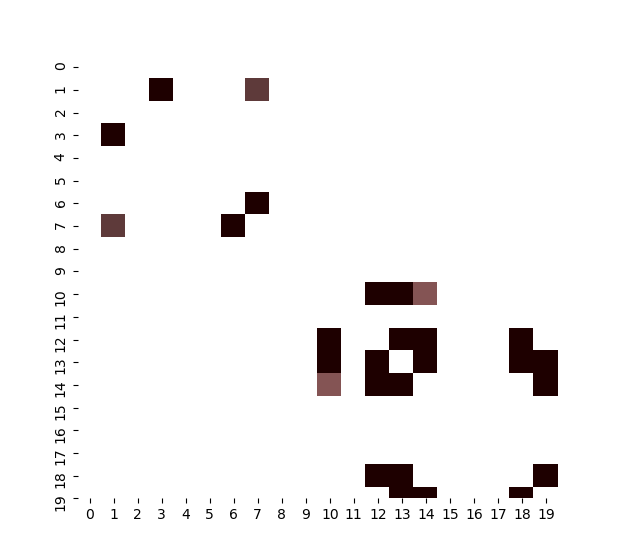}
    \subcaption{$\psi = 10^{-7}$ (APS  = 0.39)}
    \label{fig:exp_regu_graphs_spars7}
\end{subfigure}%
\begin{subfigure}[t]{0.19\textwidth}
    \centering
    \includegraphics[width = 1\linewidth]{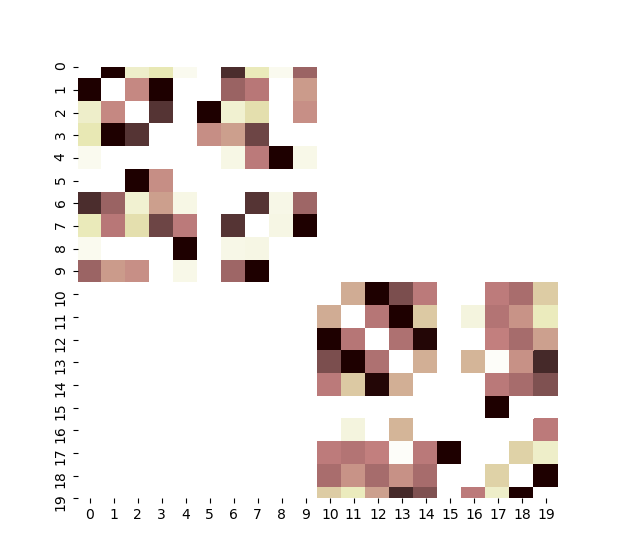}
    \subcaption{$\psi = 10^{-5}$ (APS  = 0.86)}
    \label{fig:exp_regu_graphs_spars5}
\end{subfigure}%
\begin{subfigure}[t]{0.19\textwidth}
    \centering
    \includegraphics[width = 1\linewidth]{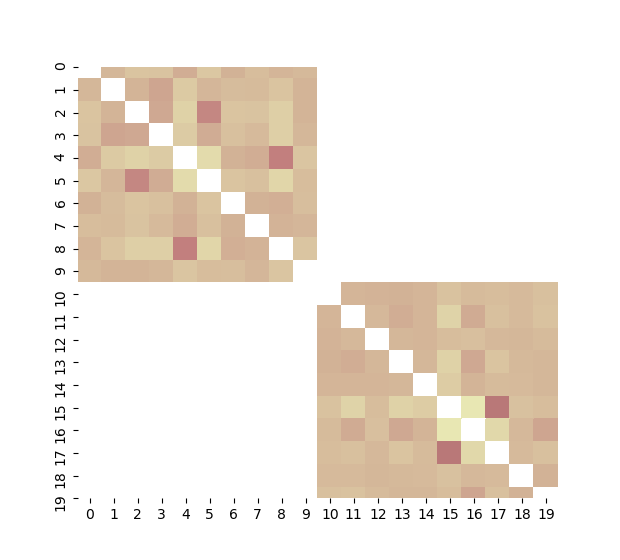}
    \subcaption{$\psi = 10^{-3}$ (APS  = 0.80)}
    \label{fig:exp_regu_graphs_spars3}
\end{subfigure}%
\begin{subfigure}[t]{0.19\textwidth}
    \centering
    \includegraphics[width = 1\linewidth]{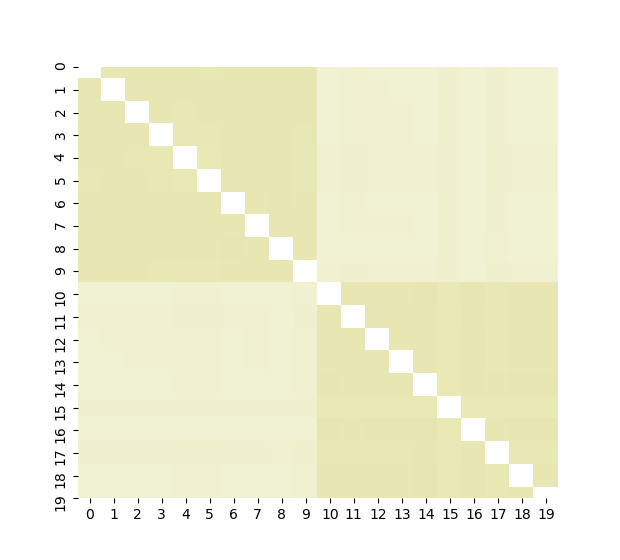}
    \subcaption{$\psi = 10^{-1}$ (APS  = 0.52)}
    \label{fig:exp_regu_graphs_spars2}
\end{subfigure}%
\begin{subfigure}[t]{0.21\textwidth}
    \centering
    \includegraphics[width = 1\linewidth]{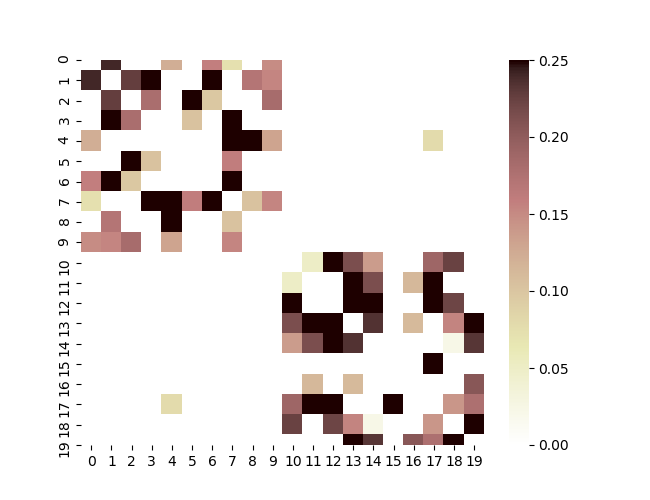}
    \subcaption{Groundtruth $\mathcal{G}_{\text{SBM}}$}
    \label{fig:exp_regu_graphs_GT}
\end{subfigure}
\begin{subfigure}[t]{0.19\textwidth}
    \centering
    \includegraphics[width = 1\linewidth]{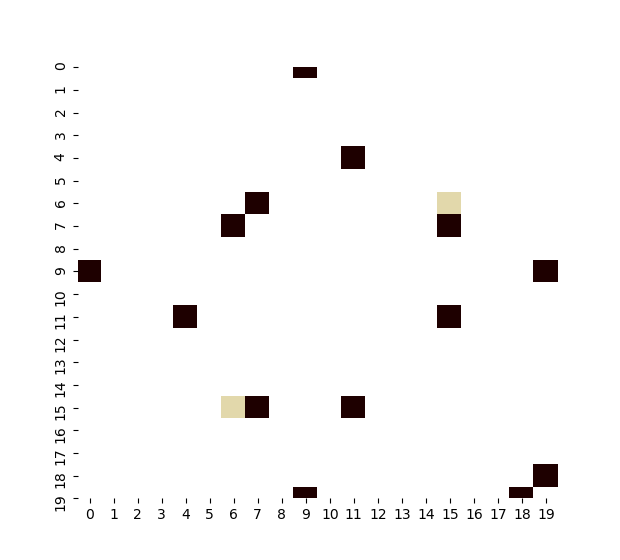}
    \subcaption{$\psi = 10^{-7}$ (APS  = 0.33)}
    \label{fig:exp_regu_graphs_spars7}
\end{subfigure}%
\begin{subfigure}[t]{0.19\textwidth}
    \centering
    \includegraphics[width = 1\linewidth]{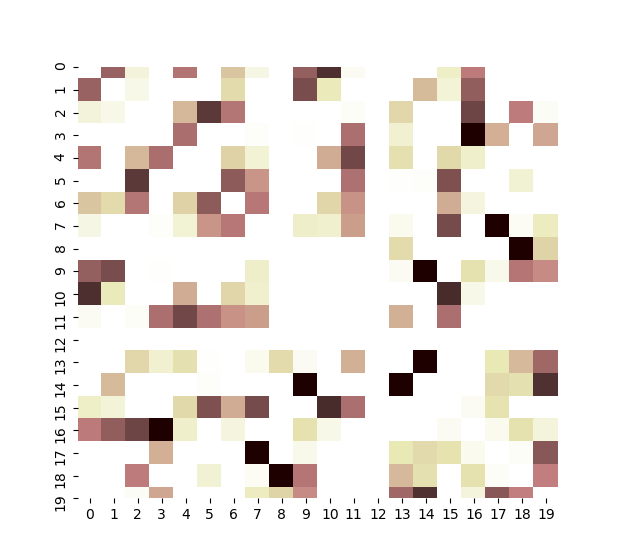}
    \subcaption{$\psi = 10^{-5}$ (APS  = 0.90)}
    \label{fig:exp_regu_graphs_spars5}
\end{subfigure}%
\begin{subfigure}[t]{0.19\textwidth}
    \centering
    \includegraphics[width = 1\linewidth]{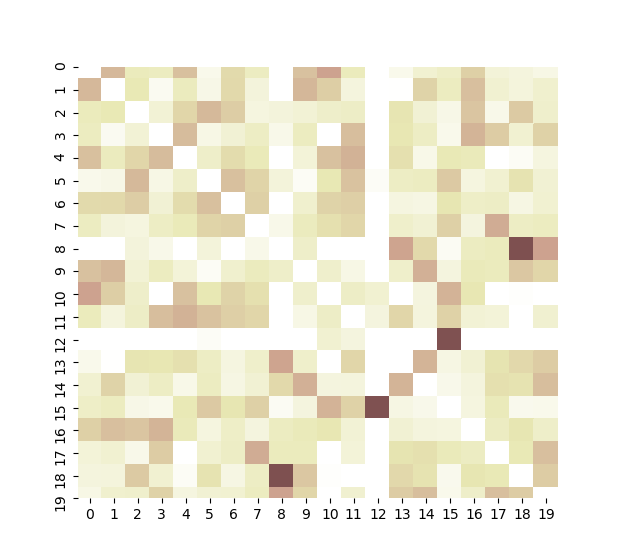}
    \subcaption{$\psi = 10^{-3}$ (APS  = 0.78)}
    \label{fig:exp_regu_graphs_spars3}
\end{subfigure}%
\begin{subfigure}[t]{0.19\textwidth}
    \centering
    \includegraphics[width = 1\linewidth]{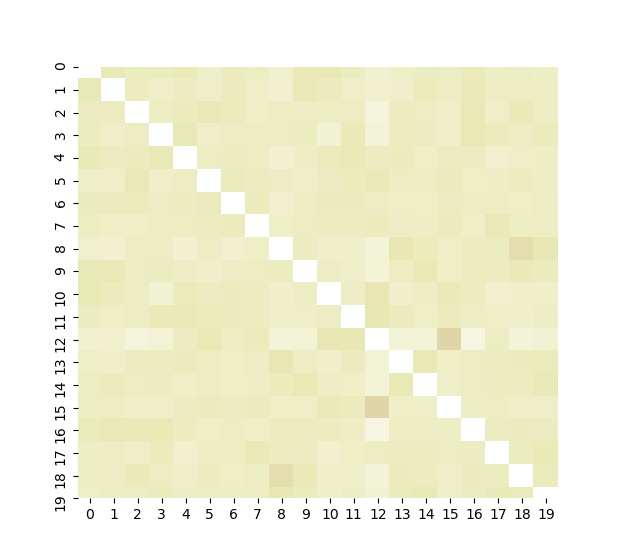}
    \subcaption{$\psi = 10^{-1}$ (APS  = 0.48)}
    \label{fig:exp_regu_graphs_spars2}
\end{subfigure}%
\begin{subfigure}[t]{0.21\textwidth}
    \centering
    \includegraphics[width = 1\linewidth]{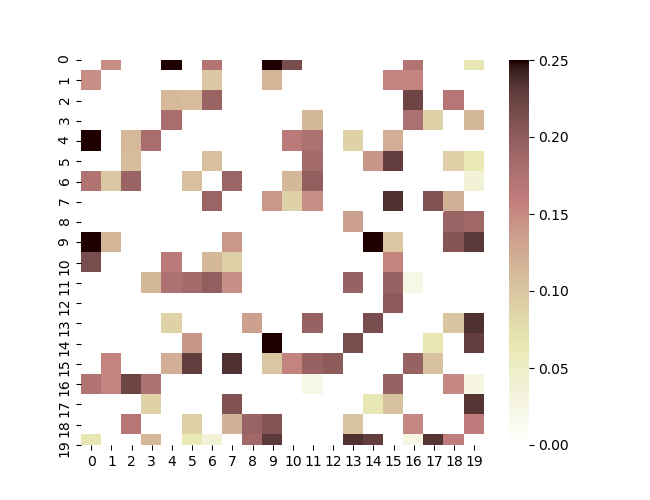}
    \subcaption{Groundtruth $\mathcal{G}_{\text{ER}}$}
    \label{fig:exp_regu_graphs_GT}
\end{subfigure}
\begin{subfigure}[t]{0.19\textwidth}
    \centering
    \includegraphics[width = 1\linewidth]{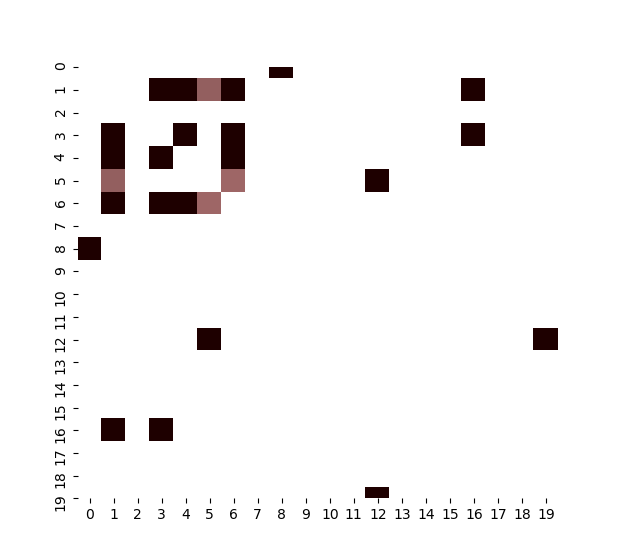}
    \subcaption{$\psi = 10^{-7}$ (APS  = 0.40)}
    \label{fig:exp_regu_graphs_spars7}
\end{subfigure}%
\begin{subfigure}[t]{0.19\textwidth}
    \centering
    \includegraphics[width = 1\linewidth]{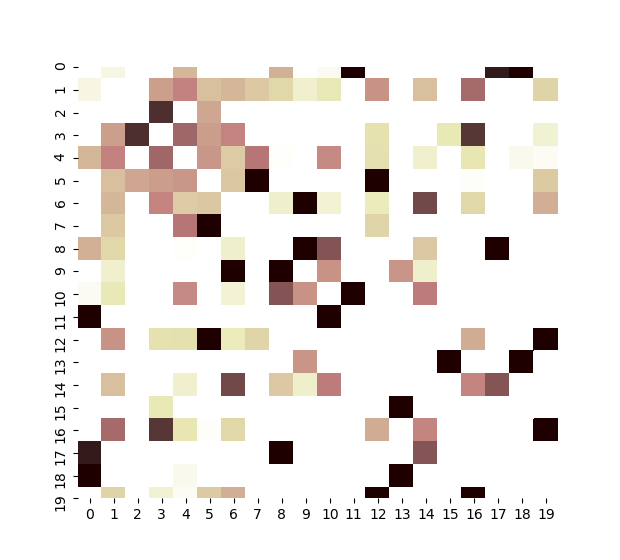}
    \subcaption{$\psi = 10^{-5}$ (APS  = 0.89)}
    \label{fig:exp_regu_graphs_spars5}
\end{subfigure}%
\begin{subfigure}[t]{0.19\textwidth}
    \centering
    \includegraphics[width = 1\linewidth]{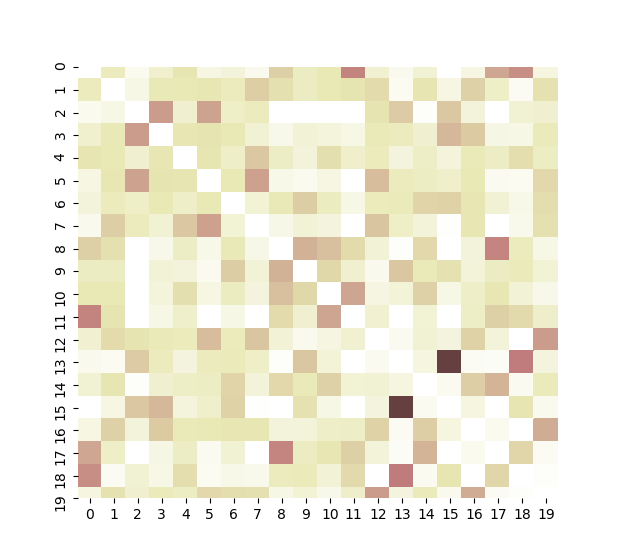}
    \subcaption{$\psi = 10^{-3}$ (APS  = 0.80)}
    \label{fig:exp_regu_graphs_spars3}
\end{subfigure}%
\begin{subfigure}[t]{0.19\textwidth}
    \centering
    \includegraphics[width = 1\linewidth]{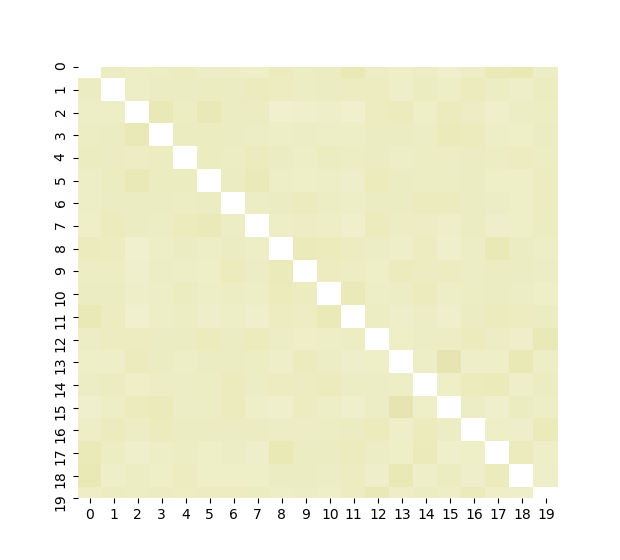}
    \subcaption{$\psi = 10^{-1}$ (APS  = 0.67)}
    \label{fig:exp_regu_graphs_spars2}
\end{subfigure}%
\begin{subfigure}[t]{0.21\textwidth}
    \centering
    \includegraphics[width = 1\linewidth]{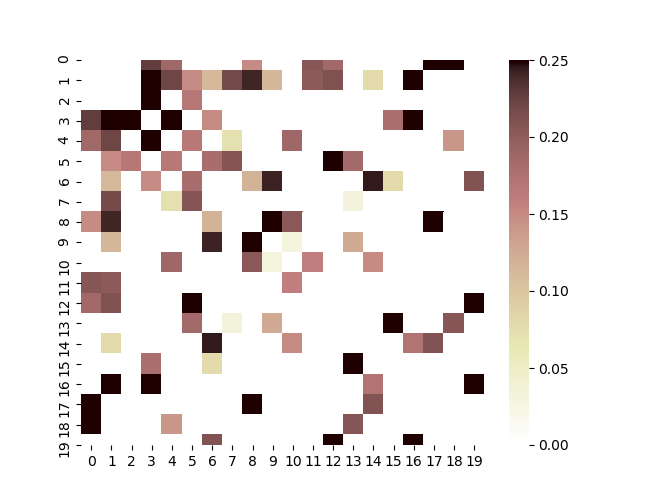}
    \subcaption{Groundtruth $\mathcal{G}_{\text{BA}}$}
    \label{fig:exp_regu_graphs_GT}
\end{subfigure}
\caption{Graph sparsity with respect to $\psi$. The first row (a)-(d): the learned $\mathcal{G}_{\text{SBM}}$; the second row (f)-(i): the learned $\mathcal{G}_{\text{ER}}$; the third row (k)-(n): the learned $\mathcal{G}_{\text{BA}}$, all from \textbf{KGL} with $\alpha = 10^{-1}$, $\rho = 10^{-2}$ and a fixed $\psi$. The respective groundtruth graphs are shown in the last column.}
\label{fig:exp_ergu_G_heatmap}
\end{figure*}

\begin{figure*}[]
\centering
\begin{minipage}[]{0.25\textwidth}
    \centering
    \includegraphics[width = 1\linewidth]{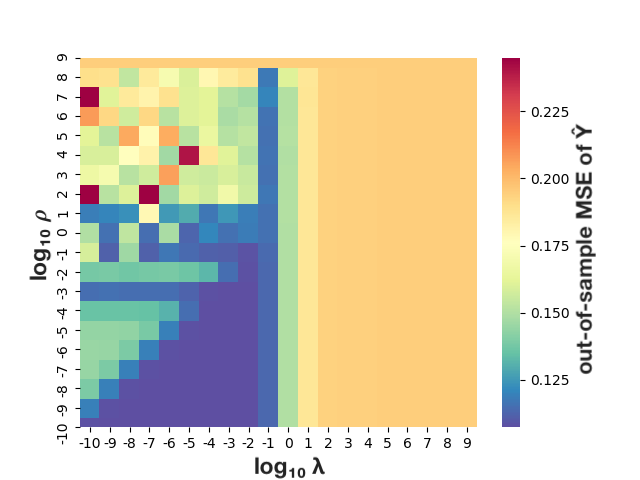}
    \subcaption{$\psi = 10^{-7}$}
    \label{fig:exp_regu_Y_ofs_psiN7}
\end{minipage}%
\begin{minipage}[]{0.25\textwidth}
    \centering
    \includegraphics[width = 1\linewidth]{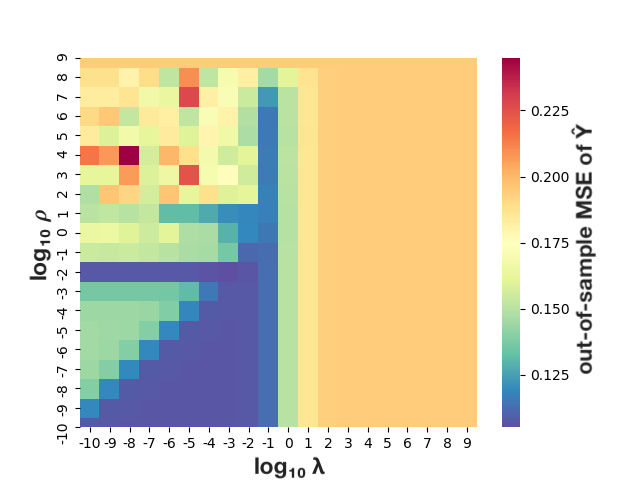}
    \subcaption{$\psi = 10^{-5}$}
    \label{fig:exp_regu_Y_ofs_psiN5}
\end{minipage}%
\begin{minipage}[]{0.25\textwidth}
    \centering
    \includegraphics[width = 1\linewidth]{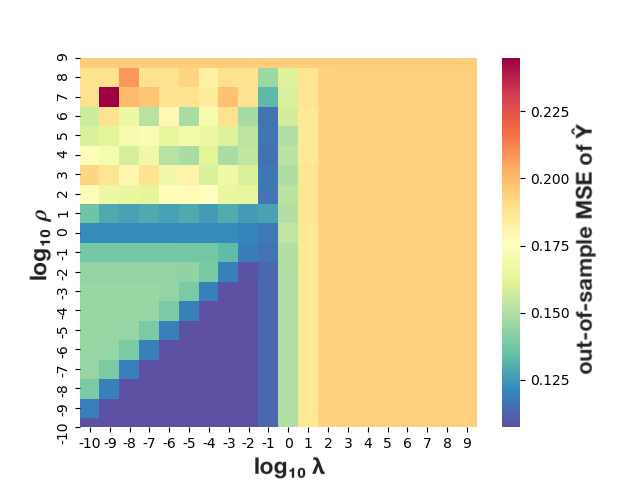}
    \subcaption{$\psi = 10^{-3}$}
    \label{fig:exp_regu_Y_ofs_psiN3}
\end{minipage}%
\begin{minipage}[]{0.25\textwidth}
    \centering
    \includegraphics[width = 1\linewidth]{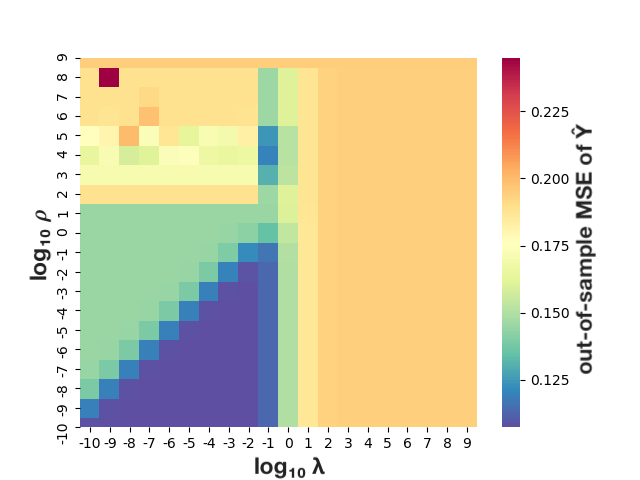}
    \subcaption{$\psi = 10^{-1}$}
    \label{fig:exp_regu_Y_ofs_psiN1}
\end{minipage}
\begin{minipage}[]{0.25\textwidth}
    \centering
    \includegraphics[width = 1\linewidth]{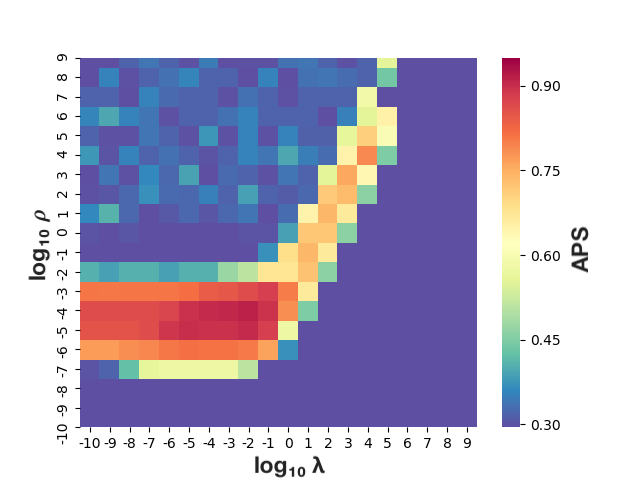}
    \subcaption{$\psi = 10^{-7}$}
    \label{fig:exp_regu_G_APS_psiN7}
\end{minipage}%
\begin{minipage}[]{0.25\textwidth}
    \centering
    \includegraphics[width = 1\linewidth]{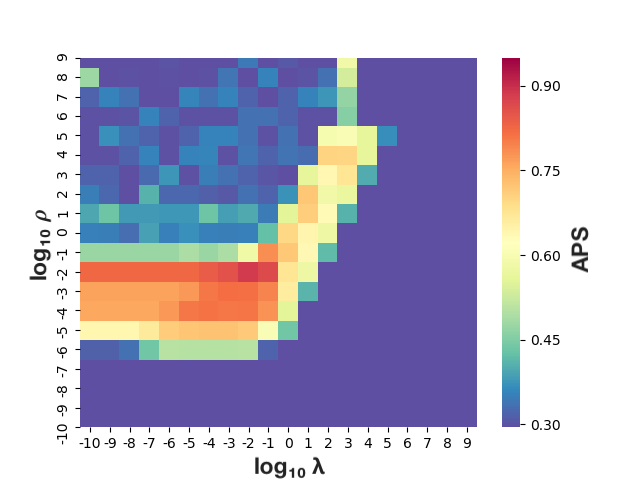}
    \subcaption{$\psi = 10^{-5}$}
    \label{fig:exp_regu_G_APS_psiN5}
\end{minipage}%
\begin{minipage}[]{0.25\textwidth}
    \centering
    \includegraphics[width = 1\linewidth]{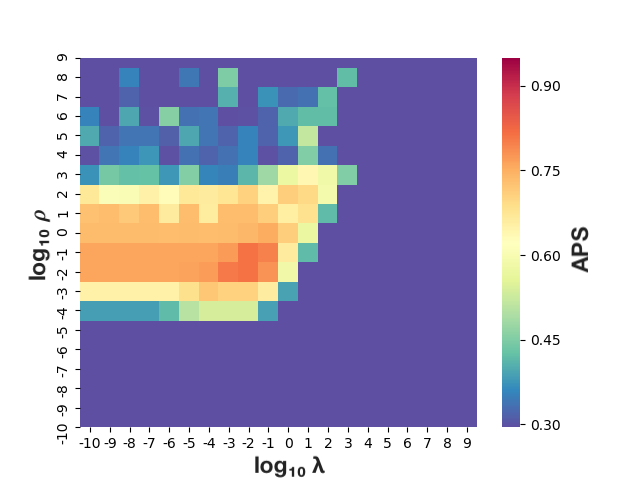}
    \subcaption{$\psi = 10^{-3}$}
    \label{fig:exp_regu_G_APS_psiN3}
\end{minipage}%
\begin{minipage}[]{0.25\textwidth}
    \centering
    \includegraphics[width = 1\linewidth]{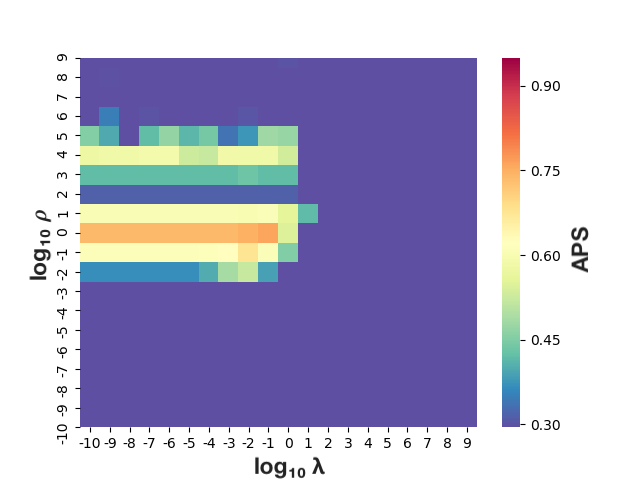}
    \subcaption{$\psi = 10^{-1}$}
    \label{fig:exp_regu_G_APS_psiN1}
\end{minipage}
\caption{The out-of-sample MSE for data matrix $\mathbf{Y}$ (the first row) and the APS of the learned graph 
(the second row) with respect to $\alpha$ and $\rho$, with 80\% entries of $\mathbf{Y}$ as training sample from \textbf{KGL} with a fixed $\psi$.}
\label{fig:exp_regularisation_path_accuracy}
\end{figure*}

\end{document}